\documentclass[draftclsnofoot, onecolumn]{IEEEtran}

\usepackage{microtype}
\usepackage{graphicx}
\graphicspath{{./}{submission/}}
\usepackage{url}
\usepackage[caption=false,font=footnotesize]{subfig}
\usepackage{booktabs}

\usepackage{amsmath}
\usepackage{amssymb}
\usepackage{mathtools}
\usepackage{amsthm}

\usepackage[capitalize,noabbrev]{cleveref}
\usepackage{aliascnt}

\theoremstyle{plain}
\newtheorem{theorem}{Theorem}[section]

\newaliascnt{proposition}{theorem}
\newtheorem{proposition}[proposition]{Proposition}
\aliascntresetthe{proposition}
\crefname{proposition}{Proposition}{Propositions}

\newaliascnt{lemma}{theorem}
\newtheorem{lemma}[lemma]{Lemma}
\aliascntresetthe{lemma}
\crefname{lemma}{Lemma}{Lemmas}

\newaliascnt{corollary}{theorem}

\aliascntresetthe{corollary}
\crefname{corollary}{Corollary}{Corollaries}

\theoremstyle{definition}
\newaliascnt{definition}{theorem}
\newtheorem{definition}[definition]{Definition}
\aliascntresetthe{definition}
\crefname{definition}{Definition}{Definitions}

\newaliascnt{assumption}{theorem}
\newtheorem{assumption}[assumption]{Assumption}
\aliascntresetthe{assumption}
\crefname{assumption}{Assumption}{Assumptions}

\theoremstyle{remark}
\newaliascnt{remark}{theorem}
\newtheorem{remark}[remark]{Remark}
\aliascntresetthe{remark}
\crefname{remark}{Remark}{Remarks}

\usepackage{bm}             
\usepackage{bbm}
\usepackage{multirow}
\usepackage{amsfonts}
\usepackage{dsfont}
\usepackage{mathrsfs}

\usepackage{tikz}
\usepackage{xcolor}

\usepackage{nicefrac}       

\usepackage{xspace}
\usepackage{enumitem}

\usepackage{comment}

\usepackage{algorithm}
\usepackage{algorithmic}

\newcommand{\PrN}{\widehat{\Pr}_N}
\newcommand{\DN}{{\clf{D}}_N}
\newcommand{\risk}{\operatorname{\mathsf{Risk}}}
\newcommand{\M}{\mathbf{M}}

\newcommand{\PPhi}{\bm{\Phi}}

\newcommand{\wrisk}{\operatorname{\mathsf{WC-Risk}}}

\newcommand{\qu}{\mathcal{Q}}

\newcommand{\B}{\mathbb{B}}

\newcommand{\Wr}{\B_r(\Pr_\circ)}

\newcommand{\Was}{\operatorname{\mathsf{W}_2}}

\theoremstyle{plain}

\theoremstyle{definition}
\newaliascnt{problem}{theorem}
\newtheorem{problem}[problem]{Problem}
\aliascntresetthe{problem}
\crefname{problem}{Problem}{Problems}

\newaliascnt{example}{theorem}

\aliascntresetthe{example}
\crefname{example}{Example}{Examples}

\theoremstyle{remark}

\DeclareMathOperator*{\argmin}{arg\,min}

\newcommand{\defeq}{\coloneqq}                      

\newcommand{\inn}{\!\in\!}
\newcommand{\+}{\!+\!}
\renewcommand{\-}{\!-\!}
\renewcommand{\=}{\!=\!}
\newcommand{\defeqq}{\!\coloneqq\!}
\renewcommand{\>}{\!>\!}

\newcommand{\ie}{\textit{i.e.}}
\newcommand{\QED}{\hfill $\blacksquare$}

\newcommand{\eps}{\varepsilon}          

\newcommand{\N}{\mathbb{N}}     
\newcommand{\R}{\mathbb{R}}     

\newcommand{\Prob}{\mathcal{P}} 

\newcommand{\pr}[1]{\left({#1}\right)}          
\newcommand{\br}[1]{\left[{#1}\right]}          
\newcommand{\cl}[1]{\left\{{#1}\right\}}        

\newcommand{\abs}[1]{\vert{#1}\vert}                    

\newcommand{\norm}[2][\text{}]{\Vert{#2}\Vert_{#1}}     
\newcommand{\Norm}[2][\text{}]{\left\Vert{#2}\right\Vert_{#1}} 

\newcommand{\clf}[1]{\mathcal{#1}} 

\newcommand{\tp}{\intercal}                 

\newcommand{\diag}{\operatorname{diag}}     

\newcommand{\vect}{\operatorname{vec}}

\newcommand{\kron}{\otimes} 

\newcommand{\E}{\operatorname{\mathds{E}}} 
\renewcommand{\Pr}{\operatorname{\mathds{P}}} 

\newcommand{\iid}{\mathrm{i.i.d.}} 

\newcommand{\sampled}[1][\text{}]{\stackrel{#1}{\sim}}

\newcommand{\indic}{\operatorname{\mathds{1}}}     

\newcommand{\beq}[1]{\begin{align*}\label{eq:#1}}
\newcommand{\eeq}{\end{align*}}

\newcommand{\suml}{\sum\nolimits}

\newcommand{\half}{\frac{1}{2}} 

\providecommand{\texorpdfstring}[2]{#1}
\newcommand{\proofnote}[1]{{\normalfont #1}}
\newcommand{\litreview}[1]{#1}

\begin{document}

\title{Distributionally Robust K-Means Clustering}

\author{Vikrant~Malik,
        Taylan~Kargin,
        and~Babak~Hassibi
\thanks{V.\ Malik, and B.\ Hassibi are with the Department
of Electrical Engineering, California Institute of Technology, Pasadena,
CA 91125 USA.}%
\thanks{T.\ Kargin was with California Institute of Technology, Pasadena, CA 91125 USA, and is now with the Massachusetts Institute of Technology, Cambridge, MA 02139 USA.}
\thanks{V.\ Malik and T.\ Kargin contributed equally to this work.}}

\maketitle
\begin{abstract}
K-means clustering is a workhorse of unsupervised learning, but it is notoriously brittle to outliers, distribution shifts, and limited sample sizes. Viewing k-means as Lloyd--Max quantization of the empirical distribution, we develop a distributionally robust variant that protects against such pathologies. We posit that the unknown population distribution lies within a Wasserstein-2 ball around the empirical distribution. In this setting, one seeks cluster centers that minimize the worst-case expected squared distance over this ambiguity set, leading to a minimax formulation. A tractable dual yields a soft-clustering scheme that replaces hard assignments with smoothly weighted ones. We propose an efficient block coordinate descent algorithm with provable monotonic decrease and local linear convergence. Experiments on standard benchmarks and large-scale synthetic data demonstrate substantial gains in outlier detection and robustness to noise.

\end{abstract}

\begin{IEEEkeywords}
Clustering, distributionally robust optimization, Wasserstein distance,
$K$-means, worst-case risk.
\end{IEEEkeywords}


\section{Introduction}

In recent years, the widespread availability of large-scale, high-dimensional datasets has driven significant interest in clustering algorithms that are both computationally efficient and robust to distributional shifts and outliers. The classical clustering method, $K$-means, can be seen as an application of the Lloyd-Max quantization algorithm, in which the distribution being quantized is the empirical distribution of the points to be clustered. This empirical distribution generally differs from the {\em true} underlying distribution, especially when the number of points to be clustered is small. This induces a distributional shift, which can also arise in many real-world settings, such as image segmentation, biological data analysis, and sensor networks, due to noise variations, sensor inaccuracies, or environmental changes. Distributional shifts can severely impact the performance of clustering algorithms, leading to degraded cluster assignments and unreliable downstream analysis.

The field of clustering has a rich history. One of the most popular algorithms in this field is the $K$-means (KM) algorithm, introduced by \cite{macqueen1967some}, which computes centroids by iteratively updating the conditional mean of the data in the Voronoi regions induced by the centroids. However, standard $K$-means is sensitive to initialization and, in general, converges only to a local minimum. The algorithm was later extended to $K$-means++ (KM++) \cite{kmeanspp}, which uses a specialized seeding method to improve its convergence characteristics. Since the underlying empirical distribution is often corrupted with noise and outliers, the development of robust clustering algorithms has been an active area of research, with notable works including \cite{georgogiannis_robust_2016, li2024adaptively, dorabiala_robust_2022, tkm, deshpande_improved_2023}.  \litreview{On the theoretical side, classical statistical analyses of $K$-means and vector quantization include Pollard's strong consistency result \cite{pollard1981strong} and nonasymptotic excess-distortion bounds for empirical vector quantization due to Levrard \cite{levrard2013fast,levrard2015nonasymptotic}. Klochkov et al. \cite{yegor2021robsut} define MOM-based robust k-means estimators and establish nonasymptotic guarantees under minimal two-moment assumptions. Their focus is statistical rather than computational: the paper does not develop or analyze a practical, efficient optimization procedure for computing these estimators.}  Noise-robust clustering has also been explored using alternative frameworks, such as the fuzzy PCA-guided approach proposed by \cite{honda2010fuzzy}. High-dimensional, low-data clustering is often encountered in biological data, as studied in \cite{zelig_kmd_2023}. In broader high-dimensional contexts, K-means has been adapted for learning feature representations \cite{montavon_learning_2012}, and the problem of clustering high-dimensional data has been explored via simultaneous feature selection \cite{liu_clustering_2023}. For an extensive survey of clustering methods, readers can refer to \cite{garcia-escudero_review_2010}.  \litreview{In practical applications, however, many robust clustering baselines, such as those considered in this paper \cite{tkm, rkm}, rely on heuristic outlier-scoring or reweighting rules rather than an explicit distributionally robust objective.} For example, \cite{tkm, rkm} identify prospective outliers by down-weighting the data points that remain farthest from the current set of tentative centers, using distance as a surrogate outlier score during training.

Fundamentally, the KM algorithm can be seen as an application of the Lloyd-Max (LM) algorithm \cite{lloyd1982least, max1960quantizing} when the distribution is the empirical distribution induced by the given data points. Lloyd \cite{lloyd1982least} introduced two iterative approaches for designing a locally optimal quantizer, referred to as Method I and Method II. Max \cite{max1960quantizing} derived necessary conditions for a quantizer to be locally optimal and independently proposed Lloyd's Method II. The generalized clustering problem of finding the optimal quantization points and regions for any given number of data points, clusters, and dimensions is NP-hard \cite{garey1982complexity}.

In this paper, we propose a theoretically motivated robust clustering algorithm. We use the Wasserstein-2 ($\Was$) distance as a measure of distributional shifts to define a family of distributions. The $\Was$-distance between distributions $\Pr_1$ and $\Pr_2$ on $\R^d$ is defined as \cite{villani_optimal_2009}
\begin{equation*}
    \Was(\Pr_1,\Pr_2)^2 \defeq {\inf_{\Pr_{12} \in \Pi(\Pr_1,\Pr_2)} \mathbb{E}_{(X,Y)\sampled\Pr_{12} } \left[ \norm{X 
    - Y}^2 \right]} 
\end{equation*}
where $\Pi(\Pr_1,\Pr_2)$ denotes the set of all joint distributions with marginals $\Pr_1$ and $\Pr_2$. Unlike other common statistical distances such as Total Variation distance, Kullback–Leibler (KL) divergence, or Hellinger distance, the $\Was$-distance incorporates information about the geometric structure of the underlying domain, making it particularly suitable for handling structured real-world data. Additionally, unlike KL divergence, it can quantify distances between distributions with different support. $\Was$-distance has recently gained popularity as a statistical distance in diverse fields such as control \cite{kargin2023wasserstein}, filtering \cite{shafieezadeh2018wasserstein}, machine learning \cite{kuhn2019wasserstein}, and data compression \cite{malik2024distributionally}.

\subsection{Contributions and Organization}
We propose a natural variant of the classic KM algorithm that is robust to distributional shifts. Specifically, the algorithm minimizes the worst-case expected squared error over the $\Was$ family of distributions (formally stated in \cref{prob:dr}). We derive the necessary conditions for the optimal placement of $K$ cluster centers that best represent the $\Was$ family of distributions for a given nominal distribution $\Pr_{\circ}$ and ambiguity radius $r$. We then propose an efficient block coordinate descent algorithm (\cref{alg:DR lloyd max}) with provable monotonic decrease and local linear convergence that alternates between solving a quadratic program for the soft assignments and applying a closed-form centroid update. Finally, we present numerical simulations to study both performance and runtime. The organization of the paper is as follows. In Section \ref{sec:prelim}, we formally define the optimization problem on which the clustering method is based. Section \ref{sec:kkt} gives the necessary conditions for optimality. Section \ref{sec:algo} formalizes the iterative clustering algorithm and provides approximation and local convergence guarantees. In Section \ref{sec:simul}, we present numerical simulations on synthetic and real-world datasets and show that, in the data-starved and corrupted regimes, our proposed algorithm outperforms standard and robust baselines \cite{kmeanspp, tkm, rkm} in terms of outlier detection and classification accuracy.

Our approach is especially suited for low-data regimes in which the empirical distribution does not represent the true distribution well, making it difficult for standard algorithms to learn the true distribution or detect outliers. It is also suited for settings in which the empirical data is corrupted with outliers. This is discussed in detail in Section \ref{sec:algo_out}. The proofs of our theorems and lemmas are included in the Appendix.

\paragraph{Notation} The letters $\N$ and $\R$ denote the sets of natural and real numbers, respectively. We denote the set of natural numbers up to a given $N\in\N$ by $[N]\defeq \{1,\dots,N\}$. The set of positive real numbers is denoted by $\R_{+}$. The set $\Delta_K\defeq \{\bm{\pi}\in\R_+^K\mid \sum_{i\in[K]} \pi_i=1\}$ denotes the $K$-probability simplex for $K\in\N$. The indicator function of a set $\Phi \subseteq \R^d$ is denoted by $\indic_{\Phi}:\R^d\to\{0,1\}$, where $\indic_{\Phi}(x)=1$ if and only if $x\in\Phi$ and $0$ otherwise. We reserve the boldface capital letters (e.g., $\mathbf{X}, \M$, etc.) for matrices. The expressions $\norm{x}$ and $x^\tp$ denote the Euclidean norm and the transpose of a vector $x\in\R^d$. We use the letters $\Pr$ and $\mathds{Q}$ to denote probability distributions, and $\E_{\Pr}$ to denote expectation under a distribution $\Pr$. The set of probability distributions over a domain $\R^d$ is denoted by $\Prob(\R^d)$, whereas $\Prob_p(\R^d)$ denotes the set of distributions with finite $p^{\textrm{th}}$ moment.

\section{Problem Setup} \label{sec:prelim}

Let $X_\sharp\sampled \Pr_\sharp$ be a random vector that encapsulates the true data population of interest, drawn from an \emph{unknown} probability distribution $\Pr_\sharp\in \Prob(\R^d)$. Our objective is to partition this population into $K\in\N$ distinct clusters. Formally, we seek a partition of the domain $\R^d$ into $K$ pairwise disjoint subsets, \ie, $\Phi_1,\dots,\Phi_K \subseteq \R^d$, such that
\[
    \Phi_k \cap \Phi_l = \emptyset \text{ for } k\neq l, \qquad \bigcup\nolimits_{k=1}^{K} \Phi_k = \R^d.
\]
Additionally, each cluster is associated with a representative center, \ie, $\mu_k\in\Phi_k$ for $k\in[K]$. For any such pair of cluster centers and partitions $(\M, \PPhi)$, where $\M\defeq \begin{bmatrix}\mu_1 &\dots &\mu_K \end{bmatrix} \in \R^{d\times K}$ and $\PPhi \defeq \{\Phi_k\}_{k=1}^{K}$, we define an associated clustering function (also called a quantizer) $\qu_{\M, \PPhi}:\R^d\to\R^d$ that maps a given data point $x\in\R^d$ to the representative center of the cluster containing $x$, namely,
\begin{equation}
    \qu_{\M, \PPhi} (x) \defeq \suml_{k=1}^{K} \mu_k \indic_{\Phi_k}( x).
\end{equation}

\subsection{Standard K-Means Clustering as Empirical Risk Minimization}
In order to assess the performance of a clustering scheme $(\M, \PPhi)$ for an arbitrary random vector $X$ following a distribution $\Pr\in\Prob(\R^d)$, we consider the expected squared $L_2$-distance between the random vector $X$ and its corresponding cluster center $\qu_{\M, \PPhi}(X)$, namely:
\begin{definition}[Clustering Risk] 
The $L_2$-risk of a clustering scheme $(\M, \PPhi)$ for a distribution $\Pr\in\Prob_2(\R^d)$ is defined as
\begin{equation}
    \risk(\M,\PPhi, \Pr) \defeq \E_{\Pr}\br{\norm{X - \qu_{\M, \PPhi}(X)}^2}. 
\end{equation} 
\end{definition}
Ideally, we seek to find a clustering scheme $(\M, \PPhi)$ that minimizes the population risk under the true but unknown population distribution $\Pr_\sharp$, namely,
\begin{equation}
    \inf_{(\M, \PPhi)} \!\risk(\M,\PPhi, \Pr_\sharp) \= \inf_{(\M, \PPhi)}\!\E_{\Pr_\sharp}\br{\norm{X_\sharp \- \qu_{\M, \PPhi}(X_\sharp)}^2} . \label{eq:opt_problem}
\end{equation} 
In reality, we almost never have access to the true distribution to begin with. Instead, usually a finite dataset $\clf{D}_N \defeq \{x_n\}_{n=1}^{N} \subset \R^d$ of $N\in\N$ points is given, presumably drawn from the underlying but unknown distribution $\Pr_\sharp$. In this case, the empirical distribution 
\begin{equation}
    \PrN \defeq N^{-1}\suml_{n=1}^{N} \delta_{x_n}
\end{equation}
of the dataset $\clf{D}_N$ can be used as a proxy for the true but unknown population distribution $\Pr_\sharp$, leading to the \emph{empirical clustering risk} minimization:
\begin{problem}[Standard $K$-Means as Empirical Risk Minimization]
Given a dataset $\clf{D}_N$, find a pair of $K$ cluster centers and regions $(\M,\PPhi)$ that minimizes the empirical clustering risk, \ie,
\begin{equation*}
    \inf_{(\M, \PPhi)} \risk(\M,\PPhi, \PrN) = \inf_{(\M, \PPhi)}\E_{\PrN}\br{\norm{X \- \qu_{\M, \PPhi}(X)}^2} = \inf_{(\M, \PPhi)}\frac{1}{N}\sum_{n=1}^{N} \norm{x_n \-  \qu_{\M, \PPhi}(x_n)}^2.
\end{equation*} 
\end{problem}
This is the standard objective used in traditional $K$-means clustering. For a fixed partition $\PPhi$, the objective is convex in the centers $\M$. However, the joint optimization over $\M$ and $\PPhi$ is highly non-convex. Nevertheless, the necessary optimality conditions are obtained as follows:
\begin{theorem}[{Optimality of $K$-Means \cite{max1960quantizing, lloyd1982least}}]\label{thm:kkt kmeans}
For a fixed set of centers $\M$, the optimal partitioning $\PPhi^\star$ is the nearest-neighbor partition, \ie,
    \begin{equation}\label{eq:nn part}
         \Phi^\star_k = \left\{ x \in \R^d \mid \norm{x - \mu_k} \leq \norm{x - \mu_l} \quad \forall  l \neq k\right\},
    \end{equation}
for $k\=1,\dots,K$. Furthermore, the optimal centers $\M^\star$ for a dataset $\clf{D}_N$ satisfy the necessary conditions below:
    \begin{equation}
        \mu^\star_k =\frac{\suml_{n=1}^{N} x_n \indic_{\Phi^{\star}_k}(x_n)}{\suml_{n=1}^{N} \indic_{\Phi^{\star}_k}(x_n) } = \E_{\PrN}\br{X\mid X\in \Phi_k^\star}.
    \end{equation}
\end{theorem}
The popular Lloyd-Max algorithm is developed as a fixed-point method for finding locally optimal centers.
\begin{gather*}
    \Phi_k^{(t)}\! \gets \!\left\{ x \inn \R^d \mid \norm{x \- \mu^{(t)}_k} \leq \norm{x \- \mu^{(t)}_l}, \; \forall  l \neq k\right\},\\
    \mu_{k}^{(t+1)} \!\gets\! \frac{\sum_{n=1}^{N} x_n \indic_{\Phi^{(t)}_k}(x_n)}{\sum_{n=1}^{N} \indic_{\Phi^{(t)}_k}(x_n) }  \= \E_{\PrN}\br{X\mid X\inn \Phi_k^{(t)}} 
\end{gather*}
The Lloyd-Max algorithm has per-iteration complexity $O(dNK)$, and the empirical risk is monotonically decreasing at each iteration, namely, $\risk(\M^{(t+1)}, \PPhi^{(t+1)},\PrN) \leq \risk(\M^{(t)}, \PPhi^{(t)},\PrN),$ converging to a local minimum. Despite its simplicity and favorable complexity, the accuracy of the standard $K$-means clustering scheme on the true population $\Pr_\sharp$ is highly sensitive to the quality of the finite dataset $\mathcal{D}_N$ and its representativeness of the true population.

\subsection{Distributionally Robust K-Means Clustering}

When the dataset $\mathcal{D}_N$ is generated by sampling directly from the population $\Pr_\sharp$, the representativeness of the empirical distribution $\PrN$ as a proxy for the population $\Pr_\sharp$ diminishes severely in high-dimensional, data-scarce settings such that $\log(N)\ll d$, due to the curse of dimensionality. Moreover, standard $K$-means does not provide an a priori performance guarantee for the underlying population when the dataset is subject to distribution shifts, making it highly susceptible to cluster misrepresentation. Examples of distributional shifts include cases in which the empirical distribution is corrupted with noise, certain clusters are under- or over-represented due to sampling bias, or the data is drawn from a skewed or shifted population relative to the true underlying distribution.

To overcome these limitations, we explicitly incorporate our uncertainty about the true population distribution $\Pr_\sharp$ into the design of $K$-means clusters through an ambiguity set of plausible probability distributions consistent with the dataset $\clf{D}_N$, to which the true distribution $\Pr_\sharp$ belongs. As motivated in the introduction, we will describe this notion of plausibility using the Wasserstein-2 distance.

\begin{definition}[Wasserstein-2 Ambiguity Sets] The $\Was$-ambiguity set of radius $r \in (0, \infty)$ around a nominal distribution $\Pr_\circ\in\Prob(\R^d)$ is defined as
\begin{equation}\label{eq:wass_ball}
    \Wr \defeq \cl{\Pr \in \Prob(\R^d)  \mid  \Was(\Pr,\, \Pr_\circ) \leq r}.
\end{equation}
\end{definition}
In this paper, we take the empirical distribution $\PrN$ as the nominal distribution. Notably, the Wasserstein-2 ambiguity set offers the significant advantage of encompassing continuous, mixed, and discrete distributions. Accordingly, we impose the following assumption for the remainder of our paper:
\begin{assumption}\label{asmp:unknown belongs}
    Given $\clf{D}_N$ and $r\>0$, the true distribution $\Pr_\sharp$ of the underlying population belongs to the ambiguity set $\mathbb{B}_r(\PrN)$.
\end{assumption}
The radius $r\>0$ essentially determines the level of uncertainty about and/or deviation from the unknown population distribution $\Pr_\sharp$. When the dataset samples are drawn $\iid$ from the population $\Pr_\sharp$, it is indeed possible to determine an upper bound on $r$ that guarantees that \cref{asmp:unknown belongs} holds with high probability.
\cref{asmp:unknown belongs} allows establishing performance bounds on the population risk of any clustering scheme through the worst-case risk attained among all plausible distributions in the ambiguity set, defined as follows:
\begin{definition}[Worst-Case Risk]
Given a dataset $\DN$ and radius $r\>0$, the worst-case risk of a quantization scheme $(\M,\PPhi)$ is defined as
\begin{equation}\label{eq:worst risk}
     \wrisk_{N,r}(\M, \PPhi) \defeq \sup_{\Pr \in \mathbb{B}_r(\PrN)}  \risk(\M,\PPhi, \Pr) = \sup_{\Pr \in \mathbb{B}_r(\PrN)}\E_{\Pr}\br{\norm{X - \qu_{\M, \PPhi}(X)}^2}.
\end{equation}
\end{definition}
Although the true distribution is not known, the worst-case risk provides an upper bound on the population risk of a clustering scheme $(\M,\PPhi)$,
\begin{equation}
    \inf_{\M^\prime,\PPhi^\prime} \risk(\M^\prime, \PPhi^\prime, \Pr_\sharp) \leq  \risk(\M, \PPhi, \Pr_\sharp) \leq \wrisk_{N,r}(\M, \PPhi).
\end{equation}
Thus, minimizing the worst-case clustering risk among all plausible distributions in the ambiguity set instead of directly minimizing the empirical risk provides a proxy with guaranteed performance in the face of uncertainties and distribution shifts between empirical training data and the true population. This is formalized in the following problem.
\begin{problem}[Distributionally Robust $K$-Means]\label{prob:dr}
    Given a dataset $\clf{D}_N$ and an ambiguity set radius $r>0$, find a clustering scheme $(\M,\PPhi)$ that minimizes the worst-case clustering risk, \ie,
\begin{equation}
    \inf_{(\M, \PPhi)} \sup_{\Pr \in \mathbb{B}_r(\PrN)} \risk(\M,\PPhi, \Pr)  = \inf_{(\M, \PPhi)} \sup_{\Pr \in \mathbb{B}_r(\PrN)}\E_{\Pr}\br{\norm{X - \qu_{\M, \PPhi}(X)}^2}.
\end{equation} 
\end{problem}
The dependence of the radius of the ambiguity set $r$ on the number of points and the problem dimension is highlighted through the following theorem.
\begin{theorem}[{High-Confidence Ambiguity Sets \cite{fournier_rate_2015}}]
\label{thm:optradius}
Let $d\>4$,  $\eps \in (0,1)$, and $\E_{\Pr_\sharp}\br{\exp(\norm{X}^\alpha)} < +\infty$ for an $\alpha>2$. When $\DN$ is drawn $\iid$ from $\Pr_\sharp$, the true distribution $\Pr_\sharp$ resides in the ambiguity set $\mathbb{B}_{r(N,\eps)}(\PrN)$ with probability at least $1-\eps$, \ie,
    \begin{equation}
        \Pr_\sharp^{\kron N}\cl{ \Was(\Pr_\sharp, \PrN) \leq r(N, \eps) } \geq 1\- \eps,
    \end{equation}
    where
    \begin{equation}
        r(N,\eps) \defeqq \begin{cases}
            \pr{\frac{\log(C/\eps)}{cN}}^{\frac{1}{d}}, &\textrm{if}\; N \geq \frac{\log(C/\eps)}{c},\\[2pt]
            \pr{\frac{\log(C/\eps)}{cN}}^{\frac{1}{\alpha}}, &\textrm{if}\; N < \frac{\log(C/\eps)}{c},
        \end{cases}
        \label{eq:approximation_nd}
    \end{equation}
    Here, $C,c>0$ are constants depending on $\alpha$ and $d$.
    
    The radius $r = r(N,\eps)$ is the Wasserstein-2 ambiguity-set radius from \eqref{eq:wass_ball}, i.e., $r$ is chosen so that $\Pr_\sharp \in \mathbb{B}_r(\PrN)$ with probability at least $1-\eps$.
\end{theorem}
The above theorem simply provides an a priori high-confidence upper bound on the $\Was$ distance between the unknown true distribution $\Pr_\sharp$ and the empirical distribution $\PrN$ sampled from $\Pr_\sharp$. Similar a priori bounds can also be derived for distribution shifts induced by feature noise, data corruption, or outlier injection.

\section{Optimality Conditions for DR-Clusters} \label{sec:kkt}
\begin{lemma}[Optimality of Nearest-Neighbor Partition]\label{lem:nn optimal}
For any fixed centroids $\M = \begin{bmatrix}\mu_1\,\dots\, \mu_K\end{bmatrix}\in\R^{d\times K}$, the nearest-neighbor partition $\PPhi^\star$ in \eqref{eq:nn part} is universally optimal for all $\Pr\in\Prob(\R^d)$:
    \begin{equation}
        \risk(\M, \PPhi, \Pr) \geq  \risk(\M, \PPhi^\star, \Pr) = \E_{\Pr}\br{\min_{k\in[K]}\norm{X- \mu_k}^2} \triangleq \risk(\M, \Pr).
    \end{equation}
    \proofnote{Proof: Appendix~\ref{app:proof nn optimal}.}
\end{lemma}
This result implies that the distributionally robust clusters for fixed centroids will be the nearest-neighbor clusters as well, \ie,
\begin{equation}\label{eq:QE}
     \inf_{(\M, \PPhi)} \sup_{\Pr \in \mathbb{B}_r(\PrN)} \risk(\M,\PPhi, \Pr) =  \inf_{\M \in \R^{d\times K}} \underbrace{\sup_{\Pr \in \mathbb{B}_r(\PrN)}  \E_{\Pr}\br{\min_{k\in[K]}\norm{X- \mu_k}^2}}_{\triangleq \wrisk_{N,r}(\M)},
\end{equation}
Here,
\[
    \wrisk_{N,r}(\M)\triangleq \sup_{\Pr \in \mathbb{B}_r(\PrN)}  \E_{\Pr}\br{\min_{k\in[K]}\norm{X\- \mu_k}^2}
\]
is the worst-case clustering risk of centroids $\M$ with nearest-neighbor clusters. In this formulation, evaluating $\wrisk_{N,r}(\M)$ entails solving an infinite-dimensional optimization problem over the space of probability distributions. By leveraging the strong duality property of distributionally robust optimization established in Theorem~1 of \cite{gao2023DRO}, we derive a tractable reformulation in \cref{thm:strong dual} that effectively reduces the problem to a single-variable optimization over the nominal empirical distribution.
\begin{theorem}[Strong Dual]\label{thm:strong dual}
The worst-case clustering risk $\wrisk_{N,r}(\M)$ incurred by 
any given set of $K$ centroids $\M\in\R^{d\times K}$ is equivalent to the following dual formulation:
\begin{align}
\wrisk_{N,r}(\M) &= \inf_{ \gamma>1} \gamma r^2 + \frac{1}{N}  \sum_{n=1}^{N} e_n(\gamma , \M), \label{eq:dual QE} \\
e_n(\gamma , \M) &\triangleq  \sup_{x\in \R^d} \min_{k\in[K]}\norm{x - \mu_k}^2 - \gamma \norm{x-x_n}^2. \label{eq:new quantization objective}
\end{align}
    \proofnote{Proof: Appendix~\ref{app:strong dual}.}
\end{theorem}
While the primal problem in \eqref{eq:QE} is reduced to a tractable finite-dimensional dual optimization in \eqref{eq:dual QE}, it requires solving a non-convex max-min optimization for each data point $x_n$. The following lemma provides a convex-concave min-max reformulation of the inner max-min problem and a tractable expression for the worst-case source distribution in terms of the nominal.
\begin{lemma}[Inner Saddle Point and the Worst-Case]\label{thm:inner saddle}
For $\gamma>1$, the objective $ e_n(\gamma , \M)$ in \eqref{eq:new quantization objective} is equivalent to the following convex-concave min-max optimization
    \begin{align}
       &\min_{\bm{\pi} \in \Delta_K } \max_{x\in \R^d} \sum_{k=1}^{K} \norm{x -\mu_k}^2\pi_k - \gamma \norm{x-x_n}^2 \label{eq:inner minmax} \\
       &\quad= \min_{\bm{\pi} \in \Delta_K }\, \sum_{k=1}^{K} \norm{x_n -\mu_k}^2 \pi_k + \frac{1}{\gamma-1} \norm{x_n-\M\bm{\pi}}^2. \label{eq:inner min pi only}
    \end{align}
Furthermore, this min-max optimization admits a saddle point $(x_n^\star,\bm{\pi}_n^\star) \in \R^d \times \Delta_K$ such that
    \begin{align}\label{eq: saddle }
         x_n^\star &= x_n+ \frac{x_n-  \M \bm{\pi}_n^\star}{\gamma-1}, \\
         \bm{\pi}_n^\star &\in \argmin_{\bm{\pi} \in \Delta_K} \suml_{k=1}^{K}\norm{x_n^\star -\mu_k}^2 \pi_k.
    \end{align}
    \proofnote{Proof: Appendix~\ref{app:inner saddle}.}
\end{lemma}
Note that we have introduced a latent variable, the optimal probability mass function (pmf) $\bm{\pi}_n^\star\!\in\! \Delta_K$, which is a function of the data point $x_n$. This pmf can be interpreted as a stochastic cluster assignment of $x_n$ to the given set of centroids $\M$, namely, $\pi_{kn}^\star$ represents the conditional probability of assigning the data point $x_n$ to the $k^\textrm{th}$ cluster. 
Observe that the term $\suml_{k=1}^{K}\norm{x_n -\mu_k}^2 \pi_{kn}$ in \eqref{eq:inner min pi only} represents the expected squared distance between the data point $x_n$ and the centroids $\M$ under the stochastic assignment rule $\bm{\pi}_n$. In contrast, $\M\bm{\pi}_n$ denotes the weighted average of the cluster centers, so that $\norm{x_n-\M\bm{\pi}_n}^2$ quantifies the bias associated with the data point $x_n$ under the stochastic cluster assignment scheme $(\M, \bm{\Pi})$, where $\bm{\Pi} \defeq \begin{bmatrix}\bm{\pi}_1& \dots & \bm{\pi}_N\end{bmatrix} \in \R^{K\times N}.$ Thus, the second term in the optimization problem \eqref{eq:inner min pi only} acts as a regularizer. This regularization, governed by $\gamma>1$, mediates the trade-off between the expected squared distance of $x_n$ to the cluster centroids and the bias introduced by the stochastic cluster assignment. 
Furthermore, $x_n^\star$ represents the worst-case data point associated with $x_n$ and the given centroids $\M$, so that the worst-case distribution becomes $\Pr^\star = N^{-1}\suml_{n=1}^{N} \delta_{x_n^\star}$. These worst-case data points are simply shifted away from the original data $x_n$ in the direction of the averaged cluster center $\M\bm{\pi}_n$, introducing a bias scaled by the inverse of $\gamma-1$. As the radius of the ambiguity set diminishes, \ie, $r\to 0$, we have $\gamma\to \infty$. In this limiting case, both the regularization and perturbation terms vanish, thereby recovering the standard $K$-means over the dataset $\DN$. We conclude this section with \cref{thm:KKT}, which formally states the main result of this section, \ie, the local optimality conditions for the distributionally robust centroids.
\begin{theorem}[{Necessary Conditions for Local Optimality}]\label{thm:KKT}
Given a fixed dataset $\DN$ and a radius $r>0$, any locally optimal centroids $\M^\star$ for the Wasserstein-2 distributionally robust $K$-means clustering \cref{prob:dr} must satisfy
\begin{equation}\label{eq:centroid-opt}
     \mu^\star_k =\frac{\suml_{n=1}^{N} x_n^\star \pi^\star_{kn} }{\suml_{n=1}^{N} \pi^\star_{kn} }, \quad k=1,\dots,K.
\end{equation}
Here we assume that every cluster is non-empty, i.e., $\suml_{n=1}^{N} \pi^\star_{kn} > 0$ for all $k$; otherwise the corresponding centroid $\mu_k^\star$ is undefined and can be removed.
Here, $\mathbf{X}^\star\defeq \begin{bmatrix}x_1^\star & \dots & x_N^\star \end{bmatrix} \in \R^{d\times N}$ and $\Pi^\star \in \R^{K\times N}$ are as in \eqref{eq: saddle }, and if the optimal dual variable is an interior point $\gamma^\star > 1$ (i.e., the optimizer of $\gamma$ is not at the boundary $\gamma = 1$), the optimal $\gamma^\star$ satisfies
\begin{equation}\label{eq: opt gamma}
    \gamma^\star = 1+ r^{-1} \sqrt{{N^{-1}}\suml_{n=1}^{N}\Norm{x_n-\M^\star \bm{\pi}^\star_n}^2}.
\end{equation}

\textbf{Remark.} Throughout this paper we assume $\gamma > 1$. This is justified as follows: at the boundary $\gamma = 1$, the inner supremum is finite only if every $x_n$ lies in $\operatorname{conv}\{\mu_1,\dots,\mu_K\}$ (see Appendix~\ref{app:strong dual}). Generically, at least one data point satisfies $x_n \notin \operatorname{conv}\{\mu_1,\dots,\mu_K\}$.
\proofnote{Proof: Appendix~\ref{app:KKT}.}
\end{theorem}

\section{Efficient Algorithm for DR-K-Means} \label{sec:algo}
In this section, we present a practical algorithm inspired by the theoretical insights developed in \cref{sec:kkt} to identify distributionally robust centroids. The core of our approach lies in formulating an iterative method that computes locally optimal centroidal regions through a fixed-point iteration \`a la $K$-means. These fixed points are characterized by the KKT conditions in \cref{thm:KKT}. We conclude this section with \Cref{thm:convergence} on the convergence of this algorithm and \Cref{sec:algo_out}, which discusses outlier detection.

\subsection{Algorithm Description}
Given a dataset $\DN$ and an initial set of centroids $\M^{(0)}\inn\R^{d\times K}$, we aim to find a set of centroids that are locally optimal for the fixed-$\gamma$ surrogate $F_\gamma(\M) \defeq \min_{\bm{\Pi}\in\Delta_K^{\times N}} J_\gamma(\M,\bm{\Pi})$, in an iterative manner. We 
first need to compute the stochastic assignment probabilities $\bm{\Pi}^{(t)}\inn\R^{K\times N}$ corresponding to the current centroids $\M^{(t)}$ by solving \eqref{eq:inner min pi only}, which is a convex quadratic program. The centroids are updated, with $\M^{(t+1)}$ being the minimizer of the objective $J_\gamma(\M,\bm{\Pi})$ for fixed $\bm{\Pi}^{(t)}$,
\begin{gather}
    \M^{(t+1)} \gets  \argmin_{\M\in \R^{d\times K}} J_\gamma(\M, \bm{\Pi}^{(t)}), \label{eq:m_update}\\
    J_\gamma(\M,\bm{\Pi})\!\triangleq\!\frac{1}{N}\sum_{n=1}^{N}\left(\sum_{k=1}^{K} \norm{x_n \-\mu_k}^2 \pi_{kn}\+\frac{1}{\gamma-1} \norm{x_n\-\M\bm{\pi}_n}^2\right)\nonumber
\end{gather}
This is an unconstrained convex quadratic objective in $\M\in \R^{d\times K}$ and admits a closed-form solution. Note that as $\gamma \rightarrow \infty$, the second term in \eqref{eq:m_update} vanishes and we \textbf{recover the classical Lloyd-Max algorithm}. In this limit,
the stochastic decoder transitions to the deterministic nearest-neighbor decoder, thereby eliminating ambiguity about the source distribution. 
For clarity, we present the algorithm for a fixed parameter $\gamma>1$.

\subsection{Pseudo-Code}
Before presenting the algorithm, we refer the reader to \cref{thm:opt pi} in Appendix~\ref{app:lemma_pi}, which derives a closed-form expression for the saddle point $(\mathbf{X}^\star, \bm{\Pi}^\star)$ under a fixed $\gamma>1$ and $\M$ in the scalar setting ($d=1$). This yields valuable insights into the structure of the stochastic assignment rule.
In particular, any data point $x_n$ residing in the interval $\Phi_{\gamma,k+\half}$ around the boundary point, $\mu_{k+\half}$, between adjacent clusters invariably yields a worst-case data point $x_n^\star$ located precisely at the boundary point $\mu_{k+\half}$. This necessitates a non-deterministic weighting of the centroids. Using these results, we present \cref{alg:DR lloyd max}, a novel distributionally robust $K$-means clustering algorithm.

\begin{algorithm}[ht]
   \caption{DR-$K$-Means Clustering}
   \label{alg:DR lloyd max}
\begin{algorithmic}[1]
   \STATE {\bfseries input:}  $\gamma\>1$, dataset $\DN$, initialization $\M^{(0)}$, convergence tolerance $\varepsilon\>0$,
   \REPEAT 
   \STATE $\bm{\pi}_n^{(t)} \gets \argmin_{\bm{\pi}_n\in \Delta_K}  \sum_{k=1}^{K} \norm{x_n - \mu_k^{(t)}}^2 \pi_{kn} + (\gamma-1)^{-1} \norm{x_n - \M^{(t)}\bm{\pi}_n}^2, \quad \forall n.$
   \STATE $\M^{(t+1)} \gets \argmin_{\M\in \R^{d\times K}}\frac{1}{N}\sum_{n=1}^{N}\left( \sum_{k=1}^{K}\norm{x_n - \mu_k}^2 \pi_{kn}^{(t)} + (\gamma-1)^{-1} \norm{x_n - \M\bm{\pi}_n^{(t)}}^2 \right)$
   \STATE Increment $t\gets t+1$
\UNTIL{$\max_{k}\norm{\mu_k^{(t+1)} - \mu_k^{(t)}}/ \max_{l}\norm{ \mu_l^{(t)}} \leq \varepsilon$}
   \STATE {\bfseries return} $\M^{(t)}$
\end{algorithmic}
\end{algorithm}

\subsection{Convergence and Approximation Guarantees}

Next, we quantify the quality of $k$-means++ seeding in \cref{thm:init} in terms of the worst-case risk and the local convergence rate of \cref{alg:DR lloyd max} in \cref{cor:T-iter}.
\begin{theorem}\label{thm:init}
Let $\M^{(0)}_{++}$ denote the centers generated by the randomized seeding procedure of $k$-means++, and $c_K = 8 (\ln K + 2)$. Then,
\begin{equation}\label{eq:init-approx}
\mathbb{E}\!\left[\wrisk_r\!\big(\M^{(0)}_{++}\big)\right] \le \Big( r + \sqrt{c_K \risk(\PrN, \M^{\mathrm{opt}}_N)} \Big)^{\!2},
\end{equation}
That is, the expected robust error of the $k$-means++ initialization is bounded by the square of $r$ plus the $\mathcal{O}(\sqrt{\log K})$-competitive empirical risk $\risk(\PrN, \M^{\mathrm{opt}}_N)$. Here, $\M^{\mathrm{opt}}_N$ denotes the optimal centroids for the empirical distribution $\PrN$, and $\risk(\PrN, \M^{\mathrm{opt}}_N)$ is the corresponding empirical risk. The bound arises by specializing Lemma~\ref{lem:risk-s} with $\Pr = \PrN$ (the center of the Wasserstein ball) and applying the $k$-means++ guarantee of \cite{kmeanspp} to the empirical cost.
\proofnote{Proof: Appendix~\ref{app:init}.}
\end{theorem}
\begin{theorem}\label{cor:T-iter}
 Fix $\gamma>1$ and denote the envelope $F_\gamma(\M) \defeq \min_{\bm{\Pi}\in\Delta_K^{\times N}} J_\gamma(\M,\bm{\Pi})$.
Let $\M^{(T)}$ denote the iterate after $T$ steps of \cref{alg:DR lloyd max}
initialized with $\M^{(0)}$. Let $\M^\star$ be  a non-degenerate strict local minimizer of $F_\gamma$ (i.e., $\nabla^2 F_\gamma(\M^\star) \succ 0$) such that each of the clusters is non-empty, and suppose the set of active centroids $\{\mu_k | \pi_{kn}^{\star} > 0\}$ is affinely independent for all $n$. Assume strict complementarity holds at $\bm{\Pi}^\star$.
Then, there exists a $\delta > 0$, $\rho_\star \in [0, 1)$, and $C \geq 1$ such that for any initialization $\M^{(0)}$ satisfying $\|\M^{(0)} - \M^{\star}\| \leq \delta$,
\begin{equation}\label{eq:T-iter-eq}
    F_\gamma(\M^{(T)}) \- F_\gamma(\M^\star) \leq C\,\rho_\star^T\bigl(F_\gamma(\M^{(0)}) \- F_\gamma(\M^\star)\bigr).
\end{equation}
\proofnote{Proof: Appendix~\ref{app:T-iter}.}
\end{theorem}

The assumption $\nabla^2 F_\gamma(\M^\star) \succ 0$ is equivalent to the SOSC ($P^\star \succ 0$) on the joint surrogate under affine independence (\cref{thm:sosc-equivalence} in Appendix~\ref{app:T-iter}). It holds generically: the clusters need well-defined cores, i.e., a majority of points receive hard (one-hot) assignments, so that the positive curvature dominates the assignment-softness penalty (see \cref{rem:geometric-sosc}). The analogous result for the full worst-case error $\wrisk(\M) = \inf_{\gamma>1} F_\gamma(\M)$ is given in \cref{thm:wc-linear-rate} below.

\begin{theorem}[Monotonic Convergence]\label{thm:convergence}
    Fix $\gamma > 1$. The iterates $\{\M^{(t)}\}_{t=0}^{\infty}$ generated by \cref{alg:DR lloyd max} monotonically decrease the envelope objective $F_\gamma(\M) \defeq \min_{\bm{\Pi}\in\Delta_K^{\times N}} J_\gamma(\M,\bm{\Pi})$, \ie, for all $t \geq 0$,
    \begin{equation}
        F_\gamma(\M^{(t+1)}) \leq  F_\gamma(\M^{(t)}).
    \end{equation}
    \proofnote{Proof: Appendix~\ref{app:convergence}.}
\end{theorem}
This monotonic decrease mirrors that of the classical Lloyd–Max algorithm for standard $K$-means. Since \cref{alg:DR lloyd max} operates at a fixed $\gamma > 1$, monotonicity is stated for $F_\gamma$. By jointly optimizing $(\bm{\Pi}, \gamma)$ in the inner step instead of fixing $\gamma$ (see \cref{alg:WC-BCD} in Appendix~\ref{app:wc-risk}), we can state the corresponding monotonicity and local linear convergence results directly for the worst-case clustering error $\wrisk(\M)$.

\begin{theorem}[Monotonic convergence of the worst-case error]\label{thm:wc-monotone}
  Consider the two-block scheme in \cref{alg:WC-BCD} (Appendix~\ref{app:wc-risk}), which jointly optimizes $(\bm\Pi, \gamma)$ in the inner step. The iterates $\{\M^{(t)}\}_{t=0}^\infty$ monotonically decrease the worst-case clustering error, \ie, for all $t \geq 0$,
  \begin{equation}
    \wrisk(\M^{(t+1)}) \leq \wrisk(\M^{(t)}).
  \end{equation}
  \proofnote{Proof: Appendix~\ref{app:wc-monotone}.}
\end{theorem}

\begin{theorem}[Local linear convergence for the worst-case error]\label{thm:wc-linear-rate}
  Consider the two-block scheme in \cref{alg:WC-BCD} (Appendix~\ref{app:wc-risk}). Let $\M^\star$ be a blockwise fixed point with corresponding optimal $(\bm\Pi^\star, \gamma^\star)$ satisfying $\gamma^\star > 1$. Suppose each cluster is non-empty, the active centroids are affinely independent, strict complementarity holds, the joint inner-block Hessian (with respect to the tangent-space assignment coordinates and $\gamma$) is positive definite, and $\nabla^2 \wrisk(\M^\star) \succ 0$. Then there exist $\delta > 0$, $\rho_r \in [0,1)$, and $C_r \geq 1$ such that for all $\norm{\M^{(0)} - \M^\star} \leq \delta$,
  \begin{equation}
    \wrisk(\M^{(T)}) \- \wrisk(\M^\star) \leq C_r\,\rho_r^T \bigl(\wrisk(\M^{(0)}) \- \wrisk(\M^\star)\bigr).
  \end{equation}
  \proofnote{Proof: Appendix~\ref{app:wc-linear-rate}.}
\end{theorem}

\subsection{Outlier Detection}
\label{sec:algo_out}
Our algorithm is particularly useful when the empirical distribution is corrupted with outliers (data points from a distribution other than the true distribution). Suppose we observe $N=n_1+n_2$ points, where
$\{X_i\}_{i=1}^{n_1}\sampled[\iid] \Pr_0$ are inliers and
$\{Y_j\}_{j=1}^{n_2}\sampled[\iid] \Pr_\star$ are outliers (independent of the inliers).
Let the corrupted empirical distribution be 
\[
\widehat{\Pr}_N \defeqq \frac{1}{N}\sum_{i=1}^{n_1}\delta_{X_i} + \frac{1}{N}\sum_{j=1}^{n_2}\delta_{Y_j}, \qquad c \defeqq \frac{n_2}{N}.
\]

\begin{theorem}\label{thm:optradius_contam}
Assume $\Pr_0$ and $\Pr_\star$ satisfy $\Was(\Pr_0, \Pr_\star)=D>0$, and that both satisfy the moment condition $\E[\exp(\|X\|^\alpha)] < \infty$ for some $\alpha > 2$. With probability at least $1-\eps$, the true inlier distribution $\Pr_0$ belongs to the Wasserstein ambiguity ball centered at the corrupted empirical measure:
\[
(\Pr_0^{\kron n_1}\!\otimes\! \Pr_\star^{\kron n_2})\!
\cl{\Was(\Pr_0,\widehat{\Pr}_N)\le r_{\mathrm{cont}}}
\;\ge\;
1\-\eps,
\]
where
\[
r_{\mathrm{cont}}(N,\eps,c) \defeq r(n_1,\eps/2) + \sqrt{c}\bigl(D + r(n_1,\eps/2) + r(n_2,\eps/2)\bigr),
\]
and $r(n,\eps)$ is the high-confidence radius from Theorem~2 of
\cite{fournier_rate_2015} (with $N$ replaced by $n_1$), namely
\[
r(n,\eps) \defeqq
\begin{cases}
\pr{\frac{\log(C_0/\eps)}{c_0\,n}}^\frac{1}{d},
&\textrm{if}\; n \ge \frac{\log(C_0/\eps)}{c_0},
\\[1mm]
\pr{\frac{\log(C_0/\eps)}{c_0\,n}}^\frac{1}{\alpha},
&\textrm{if}\; n < \frac{\log(C_0/\eps)}{c_0},
\end{cases}
\]
for constants $C_0,c_0>0$ depending on $\alpha$ and $d$.

\textbf{Remark.} The population-level assumption $\Was(\Pr_0, \Pr_\star) = D$ does not directly bound the sample-level distance $\Was(\hat{\Pr}_{n_1}, \hat{\Pr}_{n_2})$. Instead, we use the triangle inequality $\Was(\hat{\Pr}_{n_1}, \hat{\Pr}_{n_2}) \leq \Was(\hat{\Pr}_{n_1}, \Pr_0) + D + \Was(\Pr_\star, \hat{\Pr}_{n_2})$ and control the two empirical terms via Fournier--Guillin. A union bound over the two concentration events (each at level $\eps/2$) gives the stated probability $1 - \eps$.
\proofnote{Proof: Appendix~\ref{app:optradius-contam}.}
\end{theorem}

Therefore, choosing the DRKM++ ambiguity radius as $r=r_{\mathrm{cont}}(N,\eps,c)$ ensures $\Pr_0\in\mathbb{B}_r(\widehat{\Pr}_N)$ with probability at least $1-\eps$. Plugging this choice into \Cref{thm:init} bounds the robust risk of the $k$-means++ initialization under corrupted data; convergence of the subsequent iterations is then governed by \Cref{thm:convergence,cor:T-iter}.

In this setting, \textbf{outlier recall} is a popular metric for clustering algorithms \cite{tkm, rkm}. Outlier recall is the fraction of ground-truth outliers that an algorithm correctly identifies. Once a clustering algorithm has identified cluster centers, we measure the distances of each data point to their respective cluster center. These distances are then arranged in decreasing order, and, for a fixed $z$, the top $z$ data points in this list are identified as outliers, since they are farthest from their respective cluster centers. In the next section, we will use this metric to evaluate our algorithm.

\section{Numerical Experiments} \label{sec:simul}

We provide an empirical comparison of Distributionally Robust $K$-Means (DRKM++) against $K$-Means++ (KM++) \cite{kmeanspp}, trimmed $K$-means++ (TKM++) \cite{tkm}, and Robust KM++ (RKM++) \cite{rkm}. For better convergence, we initialize our algorithm with the seeding of $K$-Means++, which is non-prohibitive in the low-data regime. We use a mix of synthetic and real-world datasets. For real-world datasets, we show that in the low-data, high-dimensional regime, and in the presence of noise, our algorithm (\cref{alg:DR lloyd max}) performs better than its counterparts in terms of outlier recall and classification accuracy. The simulations were run on an M4 Pro chip with 24 GB of RAM.

\textbf{Choice of parameters.} In practice, we tune $\gamma > 1$ directly as the robustness hyper-parameter. At an interior optimum, $\gamma^\star$ and $r$ are linked by \eqref{eq: opt gamma}, so the real-data experiments are parameterized by $\gamma$ for convenience. The synthetic simulations in Appendix~\ref{app:exp} directly calibrate $r$ using the concentration-based formula of \cref{thm:optradius}.

\subsection{Real-World Datasets}\label{simuls}
We first test DRKM++ for outlier detection on the NIPS, MNIST, and SKIN datasets. These are available in the UCI repository \cite{uci}. As done in \cite{tkm}, to provide a fair comparison, we use PCA to reduce the dimensions of the given datasets before clustering. NIPS is a $5812$-dimensional dataset with $11{,}463$ samples. MNIST consists of $60{,}000$ grayscale $28 \times 28$ images, and SKIN-segmentation is a $3$-dimensional dataset. To operate in the low-data regime, we subsampled each dataset.
The first metric we consider is outlier recall (the fraction of the $z$ ground-truth outliers that an algorithm correctly identifies). We remark that TKM++ takes $z$ as input, whereas RKM++ takes $\alpha$, an outlier-dependent parameter, as input. In contrast, the DRKM++ clustering algorithm does not depend on the number of outliers---a clear advantage when the actual number of outliers is not known a priori. In our experiments, we take the best outlier recall values for RKM++ among different parameters $\alpha$.

Table~I(b) and \cref{tab:mnist_recall2} present the outlier recall for the MNIST dataset. In \cref{tab:mnist_recall2} we subsampled to $N = 500$ points and reduced the dimension to $80$ using PCA. We then randomly corrupted $2.5\%$ of the points by adding independent noise in a predefined range, as in TKM++. We chose $\gamma = 1.01$, and in all reported runs, DRKM++ identified all outliers. The second-best algorithm, TKM++, achieves an average recall of $96\%$ and sometimes misclassifies some of the outliers. In Table~I(b), we increase the number of training points to $N=1000$ and increase $\gamma$ to $1.5$. The performance of the other algorithms improves; however, for $k=10$, which is the true number of classes, we still outperform them.

Table~I(a) presents the outlier recall for the NIPS dataset when we subsample $N = 800$ points and reduce the dimension to $50$. In this case, we randomly choose $5\%$ of the points as outliers. We set $\gamma = 1.01$ and observe that in all reported runs, DRKM++ identified all outliers. For the SKIN dataset, since the data is only 3-dimensional, we did not observe any performance differences between the various algorithms. 

\begin{table}[ht]
  \caption{Outlier Recall for NIPS and MNIST Datasets}
  \centering
  \small
  \subfloat[NIPS: $N = 800$, $d = 50$, $5\%$ outliers, and $\gamma = 1.01$; DRKM++ identified all outliers in all reported runs for $k \in \{10,20\}$.]{%
    \begin{minipage}[t]{0.46\textwidth}
      \centering
      \setlength{\tabcolsep}{6pt}
      \begin{tabular}{@{}l r  |ccc @{}}
        \toprule
        \multirow{2}{*}{Method} & \multirow{2}{*}{$k$} &
          \multicolumn{3}{c}{Recall} \\
          & & Max. & Avg. & Min. \\
        \midrule
        \multirow{2}{*}{KM\texttt{++}}
            & 10 & 1 & 0.94 & 0.90 \\
            & 20 & 0.75 & 0.70 & 0.65 \\
            \midrule
        \multirow{2}{*}{TKM\texttt{++}}
            & 10 & 1 & 0.97 & 0.95 \\
            & 20 & 0.95 & 0.88 & 0.8 \\
            \midrule
        \multirow{2}{*}{RKM\texttt{++}}
            & 10 & 0.8 & 0.8 & 0.8 \\
            & 20 & 0.65 & 0.65 & 0.65 \\
            \midrule
        \multirow{2}{*}{\textbf{DRKM\texttt{++}}}
            & 10 & \textbf{1} & \textbf{1} & \textbf{1} \\
            & 20 & \textbf{1} & \textbf{1} & \textbf{1} \\
        \bottomrule
      \end{tabular}
    \end{minipage}}
  \hfill
  \subfloat[MNIST: $N = 1000$, $d = 80$, $5\%$ outliers, and $\gamma = 1.5$; DRKM++ performs best for $k = 10$.]{%
    \begin{minipage}[t]{0.46\textwidth}
      \centering
      \setlength{\tabcolsep}{6pt}
      \begin{tabular}{@{}l r | ccc @{}}
        \toprule
        \multirow{2}{*}{Method}  & \multirow{2}{*}{$k$} &
          \multicolumn{3}{c}{Recall}  \\
          & & Max. & Avg. & Min. \\
        \midrule
        \multirow{2}{*}{KM\texttt{++}}
             & 10 & 0.92 & 0.85 & 0.82 \\
             & 20 & 0.72 & 0.67 & 0.62 \\
            \midrule
        \multirow{2}{*}{TKM\texttt{++}}
             & 10 & 0.98 & 0.97 & 0.96 \\
             & 20 & 0.98 & 0.98 & 0.98 \\
            \midrule
        \multirow{2}{*}{RKM\texttt{++}}
             & 10 & 0.82 & 0.82 & 0.82 \\
             & 20 & 0.62 & 0.62 & 0.62 \\
            \midrule
        \multirow{2}{*}{\textbf{DRKM\texttt{++}}}
             & 10 & \textbf{1} & \textbf{0.98} & \textbf{0.96} \\
             & 20 & 0.94 & 0.89 & 0.86 \\
        \bottomrule
      \end{tabular}
    \end{minipage}}
\end{table}

The second metric we consider is classification accuracy. We compare against KM++ on the Fashion-MNIST and CIFAR-10 datasets. Before training, we corrupt a fraction $\nu$ of the training dataset with independent Gaussian noise. Fashion-MNIST consists of $70{,}000$ grayscale $28 \times 28$ images of clothing items. \cref{tab:fmnist_acc} presents the classification accuracy when the dataset is subsampled down to $2000$ and $4000$ points. We consider $k=3$ classes for our experiments, namely Dress, Sandal, and Bag. \cref{tab:fmnist_acc} shows that while KM++ can sometimes reach the same accuracy level as DRKM++, it is not reliable and the accuracy depends on the initialization. However, despite being seeded by the same KM++, our algorithm is more consistent across trials and achieves a higher average classification accuracy.

We run the same experiment on the CIFAR-10 dataset, which consists of $60{,}000$
RGB images. The results are included in the Appendix (\cref{tab:cifar_acc}). We use a pretrained ResNet-18 to reduce the dimension down to 512 and then to 256 using PCA. We subsample $5000$ points from the dimension-reduced training dataset. Due to a lack of data, both KM++ and DRKM++ have low classification accuracy. However, we observe that even in this extremely data-starved regime, DRKM++ performs better than KM++.

\begin{table}[t]
  \caption{Outlier Recall for the MNIST Dataset}
  \label{tab:mnist_recall2}
  \centering
  \small
  \subfloat[$N = 500$, $d = 80$, and the number of outliers is $2.5\%$ of $N$. We set $\gamma = 1.01$. In all reported runs ($k = 10$ clusters), DRKM++ identified all outliers.]{%
    \begin{minipage}[t]{\linewidth}
      \centering
      \setlength{\tabcolsep}{6pt}
      \begin{tabular}{@{}l r  |ccc  @{}}
        \toprule
        \multirow{2}{*}{Method} & \multirow{2}{*}{$k$} &
          \multicolumn{3}{c}{Recall} \\
          & & Max. & Avg. & Min. \\
        \midrule
        \multirow{1}{*}{KM\texttt{++}}
            & 10 & 0.52 & 0.46 & 0.38 \\
            \midrule
        \multirow{1}{*}{TKM\texttt{++}}
            & 10 & 1 & 0.96 & 0.91 \\
            \midrule
        \multirow{1}{*}{RKM\texttt{++}}
            & 10 & 0.3 & 0.3 & 0.3 \\
            \midrule
        \multirow{1}{*}{\textbf{DRKM\texttt{++}}}
            & 10 & \textbf{1} & \textbf{1} & \textbf{1} \\
        \bottomrule
      \end{tabular}
    \end{minipage}}
\end{table}

\begin{table}[t]
  \caption{Classification Accuracy for the Fashion-MNIST Dataset}
  \label{tab:fmnist_acc}
  \centering
  \small
  \subfloat[$N$ samples are reduced to $d = 40$ dimensions. We corrupt $40\%$ of the points with independent Gaussian noise. DRKM++ performs consistently across trials. We set $\gamma = 1.8$.]{%
    \begin{minipage}[t]{\linewidth}
      \centering
      \setlength{\tabcolsep}{6pt}
      \begin{tabular}{@{}l r | ccc @{}}
        \toprule
        \multirow{2}{*}{Method}  & \multirow{2}{*}{$N$} &
          \multicolumn{3}{c}{Accuracy}  \\ 
           & & Max. & Avg. & Min.  \\ 
        \midrule
        \multirow{1}{*}{KM\texttt{++}}
           & 2000 & 0.85 & 0.63 & 0.33 \\
            \midrule
        \multirow{1}{*}{\textbf{DRKM\texttt{++}}}
             & 2000 & \textbf{0.87} & \textbf{0.84} & \textbf{0.81} \\
        \midrule
        \multirow{1}{*}{KM\texttt{++}}
           & 4000 & 0.86 & 0.64 & 0.46 \\
            \midrule
        \multirow{1}{*}{\textbf{DRKM\texttt{++}}}
             & 4000 & \textbf{0.85} & \textbf{0.83} & \textbf{0.82} \\
        \bottomrule
      \end{tabular}
    \end{minipage}}
\end{table}

\subsection{Synthetic Datasets}
The numerical results for this subsection are provided in detail in Appendix \ref{app:exp}. In Experiment 1, we consider a Gaussian mixture model (GMM) and observe that as $N$ increases from $5$ to $50$, the difference in the worst-case performance between the DRKM and KM decreases. This is because a higher $N$ better represents the underlying data distribution, making it easier for KM to perform well. 

In Experiment 2, we consider a dataset that is corrupted by an outlier cluster. The original dataset consists of two clusters that are normally distributed. The outlier cluster is also a set of normally distributed points centered at $(8, 8)$. The misclassification error of DRKM and KM, averaged over 200 trial runs, is given in Table \ref{tab:1}. We see that DRKM performs better than KM in the low-data regime.

\subsection{Runtime Comparison}
\label{subsec:runtime}
We benchmark scalability on large-scale synthetic data. Table~\ref{tab:ls_runtime} reports outlier-detection recall/runtime together with large-scale classification accuracy/runtime for two Gaussian-mixture settings. In panel (a), we generate $N=5\cdot10^4$ points with $d=200$ and vary the number of clusters $K$; the number of outliers is $10\%$ of $N$. In panel (b), we consider a large-scale classification setting with $N=10^6$ and $d=200$.
KM\texttt{++} is fastest but not robust to data corruption; TKM\texttt{++} improves recall at a sizable runtime cost; DRKM\texttt{++} achieves perfect recall in outlier detection and stronger classification accuracy at larger scale. The original RKM\texttt{++} implementation failed in this setting, so we omit it here.

\begin{table}[t]
  \caption{Large-Scale Synthetic Experiments}
  \label{tab:ls_runtime}
  \centering
  \small
  \setlength{\tabcolsep}{6pt}
  \subfloat[Outlier detection on $N=5\cdot10^4$, $d=200$, with $K \in \{400,800\}$; outliers constitute $10\%$ of $N$ and are induced by adding noise to $10\%$ of the data.]{%
    \begin{minipage}[t]{0.48\textwidth}
      \centering
      \begin{tabular}{@{}l r | c c @{}}
        \toprule
        Method & $K$ & Recall & Time (s) \\
        \midrule
        KM\texttt{++}   & 400 & 0.96 & 0.96 \\
        TKM\texttt{++}  & 400 & 0.94 & 22.81 \\
        \textbf{DRKM\texttt{++}} & 400 & \textbf{1.00} & 1.76 \\
        \midrule
        KM\texttt{++}   & 800 & 0.84 & 1.92 \\
        TKM\texttt{++}  & 800 & 0.89 & 45.09 \\
        \textbf{DRKM\texttt{++}} & 800 & \textbf{1.00} & 4.27 \\
        \bottomrule
      \end{tabular}
    \end{minipage}}
  \hfill
  \subfloat[Classification on $N=10^6$, $d=200$, with $K \in \{200,400\}$. The dataset is a Gaussian mixture model with uniform cluster weights. We compare against the native implementation of k-means++ in MATLAB, which is heavily optimized for faster execution. TKM++ has a significantly higher runtime than DRKM++.]{%
    \begin{minipage}[t]{0.48\textwidth}
      \centering
      \begin{tabular}{@{}l r | c c @{}}
        \toprule
        Method & $K$ & Accuracy & Time (s) \\
        \midrule
        KM\texttt{++}   & 200 & 0.84 & 1.42 \\
        \textbf{DRKM\texttt{++}} & 200 & \textbf{0.95} & 17 \\
        \midrule
        KM\texttt{++}   & 400 & 0.86 & 2.05 \\
        \textbf{DRKM\texttt{++}} & 400 & \textbf{0.93} & 25 \\
        \bottomrule
      \end{tabular}
    \end{minipage}}
\end{table}

\section{Conclusion}

We introduced a distributionally robust variant of the $K$-means algorithm that explicitly accounts for distribution shifts by minimizing the worst-case clustering risk over a $\Was$ ambiguity set centered at the empirical data distribution. This ensures robustness of the cluster placements to distributional uncertainties. We established that the proposed algorithm operates as a descent method, provided approximation and local linear convergence guarantees, and derived the necessary conditions for the optimal $K$ cluster centroids. Numerical simulations confirm the efficacy of our approach, demonstrating its performance in data-starved and corrupted regimes where traditional robust algorithms may falter.

\appendices
\section{\texorpdfstring{Proof of \Cref{lem:nn optimal}}{Proof of the nearest-neighbor partition lemma}} \label{app:proof nn optimal}
Since $\PPhi$ is a partition, for any $x\in\R^d$, there is exactly one $k^\star\in[K]$ for which $\indic_{\Phi_{k^\star}}(x)=1$. Then, we have that:
\begin{align}
      \norm{x-  \qu_{\M,\PPhi}(x)}^2 &= \Norm{\suml_{k=1}^{K}\indic_{\Phi_k}(x)(x- \mu_k)}^2, \\
      &= \suml_{k=1}^{K}\indic_{\Phi_k}(x) \norm{x- \mu_k}^2,\\
      &\geq \min_{k\in[K]}\norm{x- \mu_k}^2.
\end{align}
Since this holds for any $x\in\R^d$, it also holds under the expectation operator over any distribution $\Pr$:
\begin{align}
      \E_{\Pr}\br{\norm{x-  \qu_{\M,\PPhi}(x)}^2} &\geq \E_{\Pr}\br{\min_{k\in[K]}\norm{x- \mu_k}^2}.
\end{align}
Finally, the nearest-neighbor partition for the given centroids $\M$ achieves this lower bound.
\QED

\section{\texorpdfstring{Proof of \Cref{thm:strong dual}}{Proof of the strong-duality theorem}}\label{app:strong dual}
Notice that the worst-case risk in \eqref{eq:QE} is in the form of a standard Wasserstein DRO problem:
\begin{equation}
     \sup_{P \in \Wr} \E_{X\sim P} \br{f(X\mid \M )}
\end{equation}
where $f(x\mid \M )\defeq \min_{k\in[K]}\norm{x - \mu_k}^2$ is the distance from a single point $x\in\R^d$ to the nearest cluster centroid. Applying the strong duality of Wasserstein DRO from Theorem~1 of \cite{gao2023DRO}, we obtain
\begin{align}
    &\inf_{ \gamma\geq 0} \gamma r^2+ \E_{X_\circ \sampled \PrN} \br{ \sup_{x\in\R^d} f(x\mid \M )- \gamma \norm{x-X_\circ}^2},\\
    =&\inf_{ \gamma\geq 0} \gamma r^2+ \frac{1}{N} \suml_{n=1}^{N}\br{ \sup_{x\in\R^d} \min_{k\in[K]}\norm{x - \mu_k}^2- \gamma \norm{x-x_n}^2},
\end{align}

Denote the inner supremum for data point $n$ by $e_n(\gamma, \M)$. We analyze three cases.

\paragraph{Case $\gamma < 1$} The negative quadratic $-\gamma\|x - x_n\|^2$ grows more slowly at infinity than $\min_{k\in[K]}\|x - \mu_k\|^2$, so $e_n(\gamma, \M) = +\infty$.

\paragraph{Case $\gamma = 1$} Using the simplex reformulation $\min_{k\in[K]}\|x-\mu_k\|^2 = \min_{\bm\pi\in\Delta_K}\sum_k \pi_k\|x-\mu_k\|^2$ and expanding the squares, the inner objective becomes
\[
    \sum_k \pi_k \|x-\mu_k\|^2 - \|x-x_n\|^2 = 2x^\top(x_n - \M\bm\pi) + \sum_k \pi_k\|\mu_k\|^2 - \|x_n\|^2.
\]
This is linear in $x$, so the supremum over $x\in\R^d$ is finite if and only if the coefficient $x_n - \M\bm\pi$ can be made zero for some $\bm\pi\in\Delta_K$, i.e., $x_n \in \operatorname{conv}\{\mu_1,\dots,\mu_K\}$. Otherwise $e_n(1, \M) = +\infty$.

\paragraph{Case $\gamma > 1$} By \cref{thm:inner saddle}, the inner sup-min admits a saddle point and the explicit formula gives
\[
    e_n(\gamma, \M) = \min_{\bm\pi\in\Delta_K}\biggl\{\sum_k \pi_k\|x_n - \mu_k\|^2 + \frac{1}{\gamma-1}\|x_n - \M\bm\pi\|^2\biggr\}.
\]

To verify the limit $\lim_{\gamma\downarrow 1} e_n(\gamma, \M) = e_n(1, \M)$, let $\delta_n^2 \defeq \min_{\bm\pi\in\Delta_K}\|x_n-\M\bm\pi\|^2$ denote the squared distance from $x_n$ to $\operatorname{conv}\{\mu_1,\dots,\mu_K\}$. If $\delta_n > 0$ (i.e., $x_n \notin \operatorname{conv}\{\mu_1,\dots,\mu_K\}$), then for every $\bm\pi\in\Delta_K$ the penalty term satisfies $(\gamma-1)^{-1}\|x_n-\M\bm\pi\|^2 \geq (\gamma-1)^{-1}\delta_n^2 \to +\infty$ as $\gamma\downarrow 1$, so $e_n(\gamma,\M) \to +\infty = e_n(1,\M)$. If $\delta_n = 0$, there exists $\bar{\bm\pi}\in\Delta_K$ with $\M\bar{\bm\pi}=x_n$; evaluating the formula at $\bm\pi = \bar{\bm\pi}$ gives $e_n(\gamma,\M) \leq \sum_k \bar\pi_k\|x_n-\mu_k\|^2$, and this upper bound is achieved in the limit. Hence $\lim_{\gamma\downarrow 1} e_n(\gamma,\M) = \min_{\bm\pi\in\Delta_K:\,\M\bm\pi=x_n}\sum_k\pi_k\|x_n-\mu_k\|^2 = e_n(1,\M)$.

Therefore, replacing $\gamma \geq 0$ by $\gamma > 1$ in the infimum does not change the dual value, establishing \eqref{eq:new quantization objective}.
To see this, define the dual objective $D(\gamma) \defeq \gamma r^2 + N^{-1}\sum_n e_n(\gamma, \M)$. For $\gamma < 1$ we have $D(\gamma) = +\infty$ (Case~1), so $\inf_{\gamma \geq 0} D(\gamma) = \inf_{\gamma \geq 1} D(\gamma)$. The pointwise convergence $e_n(\gamma, \M) \to e_n(1,\M)$ (proven above) gives $D(1) = \lim_{\gamma \downarrow 1} D(\gamma)$, where the limit is taken through values in $(1,\infty)$. Therefore $D(1) \geq \inf_{\gamma > 1} D(\gamma)$, and hence $\inf_{\gamma \geq 1} D(\gamma) = \inf_{\gamma > 1} D(\gamma)$.

Moreover, if at least one data point satisfies $x_n \notin \operatorname{conv}\{\mu_1,\dots,\mu_K\}$, then $D(1) = +\infty$, while $D(\gamma)$ is finite for every $\gamma > 1$ and $D(\gamma) \to +\infty$ as $\gamma \to \infty$. Since $D$ is continuous on $(1,\infty)$, the infimum is attained at some interior point $\gamma^\star > 1$.
\QED

\section{\texorpdfstring{Proof of \Cref{thm:inner saddle}}{Proof of the inner saddle-point lemma}}\label{app:inner saddle}

Note that the minimization over the indices inside the expectation in the primal problem \eqref{eq:QE} can be reexpressed as a minimization over the probability simplex:
\begin{align}
    \min_{i\in[K]}\,\|x - \mu_i\|^2 &=  \min_{\pi \in \Delta_K}\,\sum_{i\in[K]}  \|x - \mu_i\|^2 \pi_i \\
    &= \min_{\pi \in \Delta_K}\,\E_{I\sim \pi}\br{\|x - \mu_I\|^2},
\end{align}
where $I$ is a random index in $[K]$ distributed according to the discrete probability mass function $\pi\in\Delta_K$. Substituting this back into the definition in \eqref{eq:new quantization objective}, we get
\begin{align*}
    e_n(\gamma, \M)&=\sup_{x\in \R^d} \min_{\pi \in \Delta_K}\,\E_{I\sim \pi}\br{\|x \- \mu_I\|^2} \- \gamma \|x
    \-x_n\|^2.
\end{align*}
Note that the objective is strictly concave in $x\in\R^d$ for $\gamma>1$, affine in $\pi\in\Delta_K$. As the simplex $\Delta_K \subset \R^K$ is convex and compact, and $\R^d$ is convex, the $\sup$ and the $\min$ can be exchanged by Sion's minimax theorem; see Corollary~3.3 of \cite{sion_general_1958}. Furthermore, for any $d \geq 1$, the inner min-max problem in \eqref{eq:inner minmax} admits a saddle point $(x^\star, \pi^\star) \in \R^d\times \Delta_K$ such that
\begin{align}
     e_n(\gamma, \M)&= \sup_{x\in\R^d}\, \E_{I^\star\sim \pi^\star}\br{\|x - \mu_{I^\star}\|^2} - \gamma \|x-x_n\|^2, \label{eq:saddle point eqn for x}\\
     &= \min_{\pi \in \Delta_K}\, \E_{I\sim \pi}\br{\|x^\star - \mu_I\|^2} - \gamma \|x^\star-x_n\|^2. \label{eq:saddle point eqn for pi}
\end{align}
By strict concavity and smoothness of the objective \eqref{eq:saddle point eqn for x} in $x$, the supremum is achieved by the unique stationary point $x^\star\in\R^d$ at which the gradient with respect to $x$ vanishes. Furthermore, any $\pi^\star\in\Delta_K$ that achieves the minimum in \eqref{eq:saddle point eqn for pi} is a saddle point.

{To derive the formula for $x^\star$, we compute the gradient of the inner objective in \eqref{eq:saddle point eqn for x} with respect to $x$:
\[
\nabla_x \left[\sum_{k=1}^{K} \pi_k^\star \|x - \mu_k\|^2 - \gamma\|x - x_n\|^2\right]
= 2\!\left(\sum_{k=1}^{K} \pi_k^\star(x - \mu_k)\right) - 2\gamma(x - x_n).
\]
Setting the gradient to zero and using $\sum_k \pi_k^\star = 1$ gives
\[
(1 - \gamma)(x^\star - x_n) = \M\bm\pi^\star - x_n,
\]
and since $\gamma > 1$ we solve to obtain
\[
x^\star = x_n + \frac{x_n - \M\bm\pi^\star}{\gamma - 1},
\]
which is \eqref{eq: saddle }.}

The expression in \eqref{eq:inner minmax} can be further simplified by explicitly taking the supremum of the quadratic objective over $x$, resulting in the equivalent expression 
\begin{equation}
    \min_{\pi \in \Delta_K}\, \E_{I\sim \pi}\br{\|x_n- \mu_I\|^2} + \frac{1}{\gamma - 1} \|x_n-\E_{I\sim \pi}\br{\mu_I}\|^2.
\end{equation}
\QED

\section{\texorpdfstring{Proof of \Cref{thm:KKT}}{Proof of the local optimality theorem}} \label{app:KKT}
\begin{proof}
By \cref{thm:strong dual} and \cref{thm:inner saddle}, the robust clustering objective can be written as
\[
\inf_{M\in\R^{d\times K}} \inf_{\gamma>1} \inf_{\Pi\in\Delta_K^N} J_\gamma(M,\Pi),
\]
where
\[
J_\gamma(M,\Pi)
:=
\gamma r^2
+\frac1N\sum_{n=1}^N
\left[
\sum_{k=1}^K \|x_n-\mu_k\|^2 \pi_{kn}
+\frac{1}{\gamma-1}\|x_n-M\pi_n\|^2
\right].
\]
Here $\Pi=[\pi_1,\dots,\pi_N]\in\R^{K\times N}$, with each $\pi_n\in\Delta_K$.

Let $M^\star$ be a local minimizer of the robust objective, and let
$(\Pi^\star,\gamma^\star)$ be an inner minimizer associated with $M^\star$. Then
$(M^\star,\Pi^\star,\gamma^\star)$ is a local minimizer of $J_\gamma$ over
\[
\R^{d\times K}\times \Delta_K^N \times (1,\infty).
\]

For each $n$, define $x_n^\star$ by the saddle-point formula from \cref{thm:inner saddle}:
\[
x_n^\star
=
x_n+\frac{x_n-M^\star\pi_n^\star}{\gamma^\star-1}.
\tag*{\eqref{eq: saddle }}
\]

We first derive the centroid condition. Fix $k\in[K]$. Since
\[
\nabla_{\mu_k}\|x_n-\mu_k\|^2 = 2(\mu_k-x_n),
\]
and, for fixed $\pi_n^\star$,
\[
\nabla_{\mu_k}\|x_n-M\pi_n^\star\|^2
=
2\pi_{kn}^\star(M\pi_n^\star-x_n),
\]
we obtain
\begin{align*}
\nabla_{\mu_k}J_{\gamma^\star}(M^\star,\Pi^\star)
&=
\frac{2}{N}\sum_{n=1}^N
\pi_{kn}^\star
\left[
(\mu_k^\star-x_n)
+\frac{1}{\gamma^\star-1}(M^\star\pi_n^\star-x_n)
\right].
\end{align*}
By the definition of $x_n^\star$,
\[
x_n^\star
=
x_n+\frac{x_n-M^\star\pi_n^\star}{\gamma^\star-1},
\]
hence
\[
(\mu_k^\star-x_n)
+\frac{1}{\gamma^\star-1}(M^\star\pi_n^\star-x_n)
=
\mu_k^\star-x_n^\star.
\]
Therefore,
\[
\nabla_{\mu_k}J_{\gamma^\star}(M^\star,\Pi^\star)
=
\frac{2}{N}\sum_{n=1}^N \pi_{kn}^\star(\mu_k^\star-x_n^\star).
\]
Since $M^\star$ is locally optimal, the gradient must vanish, and thus
\[
\sum_{n=1}^N \pi_{kn}^\star(\mu_k^\star-x_n^\star)=0.
\]
For every $k$ such that
\[
s_k^\star:=\sum_{n=1}^N \pi_{kn}^\star>0,
\]
we conclude that
\[
\mu_k^\star
=
\frac{\sum_{n=1}^N x_n^\star \pi_{kn}^\star}{\sum_{n=1}^N \pi_{kn}^\star}.
\]
This proves \eqref{eq:centroid-opt} for every non-empty cluster. In particular, if all clusters are non-empty, then \eqref{eq:centroid-opt} holds for all $k=1,\dots,K$.

Next assume that the minimizing dual variable is an interior point, i.e., $\gamma^\star>1$.
Then the first-order condition with respect to $\gamma$ gives
\[
0
=
\partial_\gamma J_\gamma(M^\star,\Pi^\star)\big|_{\gamma=\gamma^\star}.
\]
Since
\[
\partial_\gamma J_\gamma(M,\Pi)
=
r^2-\frac1N\sum_{n=1}^N
\frac{\|x_n-M\pi_n\|^2}{(\gamma-1)^2},
\]
we obtain
\[
0
=
r^2-\frac1N\sum_{n=1}^N
\frac{\|x_n-M^\star\pi_n^\star\|^2}{(\gamma^\star-1)^2}.
\]
Rearranging,
\[
(\gamma^\star-1)^2
=
\frac{1}{r^2}\cdot
\frac1N\sum_{n=1}^N \|x_n-M^\star\pi_n^\star\|^2.
\]
Because $\gamma^\star>1$, we take the positive square root and obtain
\[
\gamma^\star
=
1+\frac1r
\sqrt{\frac1N\sum_{n=1}^N \|x_n-M^\star\pi_n^\star\|^2},
\]
which is exactly \eqref{eq: opt gamma}.
\end{proof}

\section{\texorpdfstring{Proof of \Cref{thm:init}}{Proof of the initialization theorem}}
\label{app:init}
We first prove the following useful lemma.
\begin{lemma}[Robust vs. standard risk]\label{lem:risk-s}
For any set of centers $\M$, the robust and standard $K$-means risks satisfy
\begin{equation}\label{eq:risk-s}
\risk(\PrN, \M) \le \wrisk_r(\M)
\le \Big( r + \sqrt{\risk(\PrN, \M)} \Big)^{\!2}.
\end{equation}
\end{lemma}
\begin{proof}
The first inequality holds because $\PrN \in \Wr$ and $\wrisk_r(\M) = \sup_{\Pr' \in \Wr} \risk(\Pr', \M)$. To get the second inequality, let $\Pr^\star$ denote the worst-case distribution achieving $\wrisk_r(\M)$, so $\Was(\PrN, \Pr^\star) \leq r$. Since $\Pr = \PrN$ is the center of the ball, $\Was(\PrN, \Pr^\star) \leq r$ holds directly. Let $\Pi$ be an optimal coupling between $\PrN$ and $\Pr^{\star}$ such that,
\begin{align}
    \sqrt{\E_{\Pi}[\|x - x^{\star}\|^2]} = \Was(\PrN,\Pr^\star) \leq r.
    \label{eq:wass_prev}
\end{align}
By the triangle inequality,
\begin{align}
    \min_k\|x^\star - \mu_k\| \leq \min_k\|x - \mu_k\| + \|x - x^{\star}\|.
    \label{eq:yeh}
\end{align}
Let $g(x) = \min_k\|x - \mu_k\|$. Squaring both sides of \eqref{eq:yeh} and taking expectation with respect to $\Pi$, we get,
\begin{align}
    \E_\Pi g(x^\star)^2 &\leq \E_\Pi \left(g(x) + \|x - x^{\star}\|\right)^2 \\
    &= \E_\Pi g(x)^2 + \E_\Pi\left(\|x - x^{\star}\|\right)^2 + 2\E_\Pi \left[g(x)\cdot\|x - x^{\star}\|\right]\\
    &\leq \E_\Pi g(x)^2 + \E_\Pi\left(\|x - x^{\star}\|\right)^2  + 2\sqrt{  \E_\Pi g(x)^2 \E_\Pi\|x - x^{\star}\|^2} \label{eq:cauchy-schwarz}\\
    &\leq \left( \sqrt{\E_\Pi g(x)^2} + r\right)^2,\label{eq:wass}
\end{align}
where \eqref{eq:cauchy-schwarz} is due to the Cauchy-Schwarz inequality and \eqref{eq:wass} is due to \eqref{eq:wass_prev}. Replacing $\E_\Pi g(x^\star)^2$ by $\wrisk_r(\M)$ and $\E_\Pi g(x)^2$ by $\risk(\PrN, \M)$, we get \eqref{eq:risk-s}. 
\end{proof}

To prove \cref{thm:init}, we apply \cref{lem:risk-s} (which already uses $\Pr = \PrN$). Note that $f(x) = (r + \sqrt{x})^2$ is concave for a fixed $r$. Taking the expectation on both sides of \eqref{eq:risk-s} and using Jensen's inequality, together with Theorem~3.1 of \cite{kmeanspp} applied to the empirical cost $\risk(\PrN, \cdot)$, we get \eqref{eq:init-approx}.

\section{\texorpdfstring{Proof of \Cref{cor:T-iter}}{Proof of the local linear convergence theorem}}

\label{app:T-iter}
\nocite{xu2013block}

\begin{proof}
Throughout this proof, $\gamma>1$ is fixed. Define the surrogate objective
\begin{align}
    &J_{\gamma}(\M,\bm{\Pi}) \defeq \gamma r^2 +\frac{1}{N}  \suml_{n=1}^{N}\suml_{k=1}^{K} \norm{x_n -\mu_k}^2 \pi_{kn} + \frac{1}{\gamma-1} \norm{x_n-\M\bm{\pi_n}}^2,
\end{align}
and the envelope function $F_\gamma(\M) \defeq \min_{\bm\Pi \in \Delta_K^{\times N}} J_\gamma(\M, \bm\Pi)$, so that \cref{alg:DR lloyd max} with fixed $\gamma$ performs block coordinate descent (BCD) on $J_\gamma(\M,\bm{\Pi})$, alternating between exact minimization over $\bm{\Pi}$ and exact minimization over $\M$.

\paragraph{Step 1: Active set, tangent space, and Hessian blocks}
Let $\M^\star$ be a strict local minimizer of $F_\gamma$ and let $\bm{\Pi}^{\star}$ be the corresponding minimizer of $J_\gamma(\M^\star, \cdot)$. Define the active set $A_n \defeq \{k \in [K] : \pi^{\star}_{kn} > 0\}$ for each data point $n$, and let $r_n = |A_n|$. Recall that strict complementarity is the KKT multiplier condition for the simplex subproblem: if $\tau_n^\star$ is the multiplier for $\sum_k \pi_{kn}=1$ and $\lambda_{kn}^\star \ge 0$ are the multipliers for $\pi_{kn}\ge 0$, then $\lambda_{kn}^\star > 0$ whenever $\pi_{kn}^\star = 0$ (equivalently, $\partial_{\pi_{kn}} J_\gamma(\M^\star,\bm\Pi^\star)=\tau_n^\star$ on $A_n$ and $>\tau_n^\star$ off $A_n$). Under this condition, the active set is locally stable. That is, for $\M$ sufficiently close to $\M^\star$, any nearby optimal assignment $\bm\Pi(\M)$ has the same zero/nonzero pattern as $\bm\Pi^\star$; equivalently, the sets $A_n$ do not change locally.

Define the tangent space $\mathcal{T}^\star$ to the active simplex faces, and let $U_n \in \R^{K \times (r_n - 1)}$ be an isometry onto the tangent subspace $\mathcal{T}_n^\star = \{\delta\bm\pi \in \R^K : \sum_{k \in A_n} \delta\pi_k = 0\}$, so that $U_n^T U_n = I_{r_n-1}$. The restricted joint Hessian is
\begin{equation}\label{eq:joint-hessian}
    P^\star = \begin{bmatrix}
    H_{\M\M} & H_{\M\Pi} \\
    H_{\Pi\M} & R
    \end{bmatrix}_{(\M^\star, \bm\Pi^\star)},
\end{equation}
where:
\begin{itemize}
    \item $H_{\M\M} = \frac{2}{N}(A + B) \otimes I_d$ with $A = \diag\bigl(\bigl\{\sum_n \pi_{kn}\bigr\}_k\bigr)$ and $B = \frac{1}{\gamma-1}\sum_n \bm\pi_n \bm\pi_n^T$. Non-empty clusters ensure $H_{\M\M} \succ 0$.
    \item $R = \frac{2}{N(\gamma-1)} \operatorname{blkdiag}\bigl(U_n^T \M^{\star T} \M^\star U_n\bigr)_{n=1}^N$, where each block is $(r_n-1)\times(r_n-1)$. Affine independence of active centroids ensures each block is positive definite, hence $R \succ 0$.
    \item $H_{\M\Pi} = \nabla^2_{\M\Pi} J_\gamma|_{\mathcal{T}^\star}$ is the cross-Hessian.
\end{itemize}

\paragraph{Step 2: The SOSC is equivalent to $\nabla^2 F_\gamma(\M^\star) \succ 0$}
We establish the central equivalence.

\begin{proposition}[Equivalent characterization of the SOSC]\label{thm:sosc-equivalence}
Suppose LICQ and strict complementarity hold at $(\M^\star, \bm\Pi^\star)$, and that the active centroids are affinely independent (so $R \succ 0$). Then
\begin{equation}\label{eq:sosc-equiv}
    P^\star \succ 0 \;\Longleftrightarrow\; \nabla^2 F_\gamma(\M^\star) \succ 0.
\end{equation}
\end{proposition}

\begin{proof}
By the Schur complement criterion (\cite{boyd2004convex}, Section~A.5.5), $P^\star \succ 0$ if and only if $R \succ 0$ and
\begin{equation}\label{eq:schur-sosc}
    H_{\M\M} - H_{\M\Pi}\, R^{-1}\, H_{\Pi\M} \succ 0.
\end{equation}
It remains to identify the Schur complement with the envelope Hessian $\nabla^2 F_\gamma(\M^\star)$. To do so, we apply the implicit function theorem (IFT) to the restricted stationarity equation
\[
    G(\M, \bm\Pi) \defeq \nabla_{\bm\Pi} J_\gamma(\M, \bm\Pi)\big|_{\mathcal{T}^\star} = 0.
\]
At the optimum $(\M^\star, \bm\Pi^\star)$, we have $G(\M^\star, \bm\Pi^\star) = 0$ by first-order optimality. The IFT requires that the Jacobian of $G$ with respect to $\bm\Pi$ is invertible:
\[
    \frac{\partial G}{\partial \bm\Pi} = \nabla^2_{\Pi\Pi} J_\gamma\big|_{\mathcal{T}^\star} = R.
\]
Since $R \succ 0$ (by affine independence), $R$ is invertible, and the IFT guarantees that $\bm\Pi^\star(\M)$ is an implicit function of $\M$ in a neighborhood of $\M^\star$. Moreover, since $J_\gamma$ is polynomial (hence $C^\infty$), the IFT gives that $\bm\Pi^\star(\M)$ is $C^\infty$.

The envelope is therefore $F_\gamma(\M) = J_\gamma(\M, \bm\Pi^\star(\M))$. Differentiating twice and substituting $G(\M, \bm\Pi^\star(\M)) = 0$ yields the classical sensitivity formula \cite[Theorem~4.24]{bonnans2013perturbation}:
\begin{equation}\label{eq:envelope-hessian}
    \nabla^2 F_\gamma(\M^\star) = H_{\M\M} - H_{\M\Pi}\, R^{-1}\, H_{\Pi\M}.
\end{equation}
Comparing \eqref{eq:envelope-hessian} with \eqref{eq:schur-sosc}, the second condition is precisely $\nabla^2 F_\gamma(\M^\star) \succ 0$. Since $R \succ 0$ by assumption, the equivalence follows.
\end{proof}

\paragraph{Step 3: Direct contraction of the one-step map}
Let $m \defeq \vect(\M) \in \R^{dK}$ and $m^\star \defeq \vect(\M^\star)$; throughout, we identify $\M \in \R^{d \times K}$ with its vectorization $m$. Let $q \defeq \sum_{n=1}^{N}(r_n - 1)$ and define a reduced coordinate $z = (z_1, \dots, z_N) \in \R^q$ via
\[
    \vect(\bm\Pi(z)) \defeq \vect(\bm\Pi^\star) + V z,
\]
where $V = \diag(U_1, \dots, U_N)$. For $z$ near $0$, $\bm\Pi(z)$ stays on the same simplex faces as $\bm\Pi^\star$. Define the reduced objective
\[
    \widehat{J}_\gamma(m, z) \defeq J_\gamma\bigl(\M, \bm\Pi(z)\bigr),
\]
viewed as a function of $(m, z) \in \R^{dK} \times \R^q$, and let
\[
    H_{mm} \defeq \nabla^2_{mm} \widehat{J}_\gamma(m^\star, 0), \quad H_{mz} \defeq \nabla^2_{mz} \widehat{J}_\gamma(m^\star, 0), \quad H_{zz} \defeq \nabla^2_{zz} \widehat{J}_\gamma(m^\star, 0).
\]
Note that $H_{mm} = H_{\M\M} \succ 0$ and $H_{zz} = R \succ 0$.

By the IFT (Step~2), there is a unique $C^\infty$ map $m \mapsto z_\gamma(m)$ near $m^\star$ with $z_\gamma(m^\star) = 0$, solving $\nabla_z \widehat{J}_\gamma(m, z_\gamma(m)) = 0$. Differentiating both sides with respect to $m$ at $m = m^\star$ via the chain rule gives
\begin{equation}\label{eq:Dz-gamma}
    \frac{dz_\gamma}{dm}\bigg|_{m=m^\star} = -H_{zz}^{-1} H_{zm}.
\end{equation}
Similarly, the centroid-update map $\mathcal{M}_\gamma(z) \defeq \argmin_m \widehat{J}_\gamma(m, z)$ satisfies
\begin{equation}\label{eq:DM-gamma}
    \frac{d\mathcal{M}_\gamma}{dz}\bigg|_{z=0} = -H_{mm}^{-1} H_{mz}.
\end{equation}
The one-step alternating-minimization map is $T_\gamma(m) \defeq \mathcal{M}_\gamma(z_\gamma(m))$, a composition of the two maps above. By the chain rule,
\[
    \frac{dT_\gamma}{dm}\bigg|_{m=m^\star} = \frac{d\mathcal{M}_\gamma}{dz}\bigg|_{z=0} \cdot \frac{dz_\gamma}{dm}\bigg|_{m=m^\star} = (-H_{mm}^{-1} H_{mz})(-H_{zz}^{-1} H_{zm}),
\]
which gives
\begin{equation}\label{eq:DT-gamma}
    \frac{dT_\gamma}{dm}\bigg|_{m=m^\star} = H_{mm}^{-1} H_{mz} H_{zz}^{-1} H_{zm} =: H_{mm}^{-1} S.
\end{equation}
By the Schur complement identity (cf.\ \eqref{eq:envelope-hessian}), $H_{mm} - S = \nabla^2 F_\gamma(\M^\star) \succ 0$, so $0 \preceq S \prec H_{mm}$. In the $H_{mm}$-inner product, $\frac{dT_\gamma}{dm}\big|_{m^\star} = H_{mm}^{-1}S$ is self-adjoint with operator norm
\[
    q_0 \defeq \bigl\|\tfrac{dT_\gamma}{dm}\big|_{m^\star}\bigr\|_{H_{mm}} = \lambda_{\max}(H_{mm}^{-1/2} S\, H_{mm}^{-1/2}) < 1.
\]
Since $T_\gamma$ is $C^1$, the map $m \mapsto \|\frac{dT_\gamma}{dm}\|_{H_{mm}}$ is continuous and equals $q_0 < 1$ at $m = m^\star$. Therefore, for any $\rho_\star \in (q_0, 1)$, we can shrink the neighborhood $\mathcal{B}$ of $m^\star$ so that $\|\frac{dT_\gamma}{dm}\|_{H_{mm}} \leq \rho_\star$ for all $m \in \mathcal{B}$. Since $T_\gamma(m^\star) = m^\star$, the mean-value theorem gives the one-step bound
\[
    \|m^{(t+1)} - m^\star\|_{H_{mm}} = \|T_\gamma(m^{(t)}) - T_\gamma(m^\star)\|_{H_{mm}} \leq \rho_\star\, \|m^{(t)} - m^\star\|_{H_{mm}}.
\]
Iterating this inequality $t$ times yields
\begin{equation}\label{eq:iterate-contraction}
    \|m^{(t)} - m^\star\|_{H_{mm}} \leq \rho_\star^t \|m^{(0)} - m^\star\|_{H_{mm}}.
\end{equation}

\paragraph{Step 4: Objective-value rate}
Since $F_\gamma(\M) = J_\gamma(\M, \bm\Pi^\star(\M))$ is a composition of $C^\infty$ maps (Step~2), it is $C^\infty$ near $\M^\star$. With $\nabla F_\gamma(\M^\star) = 0$ and $\nabla^2 F_\gamma(\M^\star) \succ 0$, there exist constants $0 < c_1 \leq c_2 < \infty$ and a (possibly smaller) neighborhood of $\M^\star$ such that
\begin{equation}\label{eq:quad-sandwich}
    c_1\, \|m - m^\star\|_{H_{mm}}^2 \leq F_\gamma(\M) - F_\gamma(\M^\star) \leq c_2\, \|m - m^\star\|_{H_{mm}}^2
\end{equation}
for all $m$ in that neighborhood. (The constants $c_1, c_2$ arise from the eigenvalues of $H_{mm}^{-1/2}\nabla^2 F_\gamma(\M^\star)\,H_{mm}^{-1/2}$, which is well defined since both $H_{mm}$ and $\nabla^2 F_\gamma(\M^\star)$ are positive definite.) The sandwich \eqref{eq:quad-sandwich} is applied only at the two endpoints $m^{(T)}$ and $m^{(0)}$---not at every intermediate iterate---so the constant appears only once:
\begin{align*}
    F_\gamma(\M^{(T)}) - F_\gamma(\M^\star)
    &\leq c_2\, \|m^{(T)} - m^\star\|_{H_{mm}}^2  &&\text{(upper bound of \eqref{eq:quad-sandwich} at step $T$)} \\
    &\leq c_2\, \rho_\star^{2T}\, \|m^{(0)} - m^\star\|_{H_{mm}}^2  &&\text{(iterate contraction \eqref{eq:iterate-contraction})} \\
    &\leq \frac{c_2}{c_1}\, \rho_\star^{2T}\, \bigl(F_\gamma(\M^{(0)}) - F_\gamma(\M^\star)\bigr). &&\text{(lower bound of \eqref{eq:quad-sandwich} at step $0$)}
\end{align*}
Setting $\rho_\star \defeq q_0^2 \in [0,1)$ and $C \defeq c_2/c_1 \geq 1$ yields \eqref{eq:T-iter-eq}. When $\nabla^2 F_\gamma(\M^\star) \not\succ 0$, the KL property still holds for some $\theta \in (1/2, 1)$ (since $J_\gamma$ is real analytic), yielding sublinear convergence $O(t^{-(1-\theta)/(2\theta-1)})$.
\end{proof}

\begin{remark}[Interpretation of the condition $\nabla^2 F_\gamma(\M^\star) \succ 0$]\label{rem:geometric-sosc}\mbox{}

At a local optimum $\M^\star$, the Schur-complement formula \eqref{eq:envelope-hessian} gives
$\nabla^2 F_\gamma(\M^\star) = H_{\M\M} - H_{\M\Pi}\, R^{-1}\, H_{\Pi\M}$,
where $H_{\M\M} \succ 0$ (non-empty clusters) and $R \succ 0$ (affine independence).
Every data point contributes positive curvature to $H_{\M\M}$, whereas the cross-Hessian
$H_{\M\Pi}$ is built from the tangent-space isometries $U_n\in\R^{K\times(r_n-1)}$.
When $r_n=1$ (hard assignment), $U_n$ is $K\times 0$, so point~$n$ does not enter
$H_{\M\Pi}$ at all; only boundary points ($r_n\geq 2$) contribute to the correction
$H_{\M\Pi}\,R^{-1}\,H_{\Pi\M}$.
Hence $\nabla^2 F_\gamma(\M^\star)\succ 0$ holds whenever most assignments are hard.

More precisely, for a fixed active-set pattern, the inner solution map
$\M \mapsto \bm\Pi^\star(\M)$ is $C^\infty$ (indeed rational on the region where the relevant Jacobian is invertible).
Hence $\M \mapsto \det\!\bigl(\nabla^2 F_\gamma(\M)\bigr)$ is real-analytic on that region.
If it is not identically zero, its zero set has Lebesgue measure zero.
Therefore $\nabla^2 F_\gamma(\M^\star) \succ 0$ is a generic condition within each active-set region.
\end{remark}

\section{\texorpdfstring{Proof of \Cref{thm:convergence}}{Proof of the fixed-gamma convergence theorem}}\label{app:convergence}
\begin{proof}
Because $\bm\Pi^{(t)}$ minimizes $\bm\Pi \mapsto J_\gamma(\M^{(t)}, \bm\Pi)$, we have
\[
    F_\gamma(\M^{(t)}) = J_\gamma(\M^{(t)}, \bm\Pi^{(t)}).
\]
Because $\M^{(t+1)}$ minimizes $\M \mapsto J_\gamma(\M, \bm\Pi^{(t)})$,
\[
    J_\gamma(\M^{(t+1)}, \bm\Pi^{(t)}) \leq J_\gamma(\M^{(t)}, \bm\Pi^{(t)}).
\]
By the definition of the envelope,
\[
    F_\gamma(\M^{(t+1)}) = \min_{\bm\Pi} J_\gamma(\M^{(t+1)}, \bm\Pi) \leq J_\gamma(\M^{(t+1)}, \bm\Pi^{(t)}).
\]
Combining the three displayed equations yields $F_\gamma(\M^{(t+1)}) \leq F_\gamma(\M^{(t)})$. Since $F_\gamma(\M) \geq \gamma r^2$ for all $\M$, the monotone nonincreasing sequence $\{F_\gamma(\M^{(t)})\}$ converges.
\end{proof}

\section{Worst-Case Error: Algorithm and Convergence}\label{app:wc-risk}

The fixed-$\gamma$ algorithm (\cref{alg:DR lloyd max}) provides convergence guarantees for the envelope $F_\gamma$. To state convergence directly for the worst-case clustering error $\wrisk(\M) = \inf_{\gamma>1} F_\gamma(\M)$, we optimize $\gamma$ jointly with $\bm\Pi$ in the inner step. This yields the following two-block scheme.
\begin{algorithm}[H]
   \caption{DR-$K$-Means with Joint $(\bm\Pi, \gamma)$ Optimization}
   \label{alg:WC-BCD}
\begin{algorithmic}[1]
   \STATE {\bfseries input:}  $r\>0$, dataset $\DN$, initialization $\M^{(0)}$, convergence tolerance $\varepsilon\>0$
   \REPEAT
   \STATE $(\bm\Pi^{(t)}, \gamma^{(t)}) \gets \argmin_{\bm\Pi \in \Delta_K^{\times N},\, \gamma > 1}\; \gamma r^2 + \frac{1}{N}\suml_{n=1}^{N}\br{\suml_{k=1}^{K} \norm{x_n - \mu_k^{(t)}}^2 \pi_{kn} + \frac{1}{\gamma-1}\norm{x_n - \M^{(t)}\bm\pi_n}^2}$
   \STATE $\M^{(t+1)} \gets \argmin_{\M \in \R^{d \times K}}\; \gamma^{(t)} r^2 + \frac{1}{N}\suml_{n=1}^{N}\br{\suml_{k=1}^{K} \norm{x_n - \mu_k}^2 \pi_{kn}^{(t)} + \frac{1}{\gamma^{(t)}-1}\norm{x_n - \M\bm\pi_n^{(t)}}^2}$
   \STATE Increment $t\gets t+1$
\UNTIL{$\max_{k}\norm{\mu_k^{(t+1)} \!-\! \mu_k^{(t)}}/ \max_{l}\norm{ \mu_l^{(t)}} \leq \varepsilon$}
   \STATE {\bfseries return} $\M^{(t)}$
\end{algorithmic}
\end{algorithm}

Note that for fixed $\bm\Pi$, the optimal $\gamma$ is given in closed form by \eqref{eq: opt gamma}; hence the joint inner minimization reduces to minimizing over $\bm\Pi$ alone at the corresponding $\gamma^\star(\bm\Pi, \M)$.

\subsection{\texorpdfstring{Proof of \Cref{thm:wc-monotone}}{Proof of the worst-case monotonicity theorem}}\label{app:wc-monotone}
\begin{proof}
Define $\widetilde{J}(\M, \bm\Pi, \gamma) \defeq \gamma r^2 + N^{-1}\suml_{n=1}^{N}\bigl[\suml_{k=1}^{K}\norm{x_n - \mu_k}^2\pi_{kn} + (\gamma-1)^{-1}\norm{x_n - \M\bm\pi_n}^2\bigr]$, so that $\wrisk(\M) = \min_{\bm\Pi,\,\gamma>1} \widetilde{J}(\M, \bm\Pi, \gamma)$.

By definition of the exact inner step,
\[
    \wrisk(\M^{(t)}) = \widetilde{J}(\M^{(t)}, \bm\Pi^{(t)}, \gamma^{(t)}).
\]
By the $\M$-step,
\[
    \widetilde{J}(\M^{(t+1)}, \bm\Pi^{(t)}, \gamma^{(t)}) \leq \widetilde{J}(\M^{(t)}, \bm\Pi^{(t)}, \gamma^{(t)}).
\]
By the definition of the envelope,
\[
    \wrisk(\M^{(t+1)}) = \min_{\bm\Pi,\,\gamma>1} \widetilde{J}(\M^{(t+1)}, \bm\Pi, \gamma) \leq \widetilde{J}(\M^{(t+1)}, \bm\Pi^{(t)}, \gamma^{(t)}).
\]
Combining these three inequalities yields $\wrisk(\M^{(t+1)}) \leq \wrisk(\M^{(t)})$. Since $\widetilde{J}(\M, \bm\Pi, \gamma) \geq \gamma r^2 \geq r^2$, the sequence is bounded below and converges.
\end{proof}

\subsection{\texorpdfstring{Proof of \Cref{thm:wc-linear-rate}}{Proof of the worst-case linear-rate theorem}}\label{app:wc-linear-rate}
\begin{proof}
The proof follows the same direct contraction strategy as the proof of \cref{cor:T-iter}, with the inner variable enlarged from $z$ (tangent-space assignment coordinates) to $u \defeq (z, \eta)$ where $\eta \defeq \gamma - \gamma^\star$.

\paragraph{Step 1: Smooth inner minimizer map}
By the same active-set argument used in the proof of \cref{cor:T-iter} (Step~2), strict complementarity ensures that the active sets stay fixed for $\M$ near $\M^\star$. The Hessian of $\widetilde{J}$ with respect to $u = (z, \eta)$ at $(\M^\star, \bm\Pi^\star, \gamma^\star)$ is positive definite by hypothesis. Therefore, by the implicit function theorem (and since $\widetilde{J}$ is polynomial, hence $C^\infty$), there exists a unique $C^\infty$ map
\[
    \M \mapsto (\bm\Pi_r(\M), \gamma_r(\M))
\]
from a neighborhood of $\M^\star$ into a neighborhood of $(\bm\Pi^\star, \gamma^\star)$, and the robust envelope is locally
\[
    \wrisk(\M) = \widetilde{J}(\M, \bm\Pi_r(\M), \gamma_r(\M)).
\]

\paragraph{Step 2: Derivative of the one-step map}
Using the same vectorization $m = \vect(\M)$ and reduced coordinates $z$ and $\bm\Pi(z)$ introduced in the proof of \cref{cor:T-iter} (Step~3), set $u \defeq (z, \eta) \in \R^{q+1}$ where $\eta \defeq \gamma - \gamma^\star$. Define the reduced lifted objective
\[
    \widehat{J}_{\mathrm{red}}(m, u) \defeq \widetilde{J}\bigl(\M, \bm\Pi(z), \gamma^\star + \eta\bigr),
\]
viewed as a function of $(m, u) \in \R^{dK} \times \R^{q+1}$, and let
\[
    H_{mm} \defeq \nabla^2_{mm} \widehat{J}_{\mathrm{red}}(m^\star, 0), \quad H_{mu} \defeq \nabla^2_{mu} \widehat{J}_{\mathrm{red}}(m^\star, 0), \quad H_{uu} \defeq \nabla^2_{uu} \widehat{J}_{\mathrm{red}}(m^\star, 0).
\]
The one-step map $T_r(m) \defeq \mathcal{M}_r(u_r(m))$ has derivative
\[
    \frac{dT_r}{dm}\bigg|_{m=m^\star} = H_{mm}^{-1} H_{mu} H_{uu}^{-1} H_{um} =: H_{mm}^{-1} S_r.
\]

\paragraph{Step 3: Schur complement identity and contraction}
Differentiating the envelope identity $\wrisk(\M) = \widehat{J}_{\mathrm{red}}(m, u_r(m))$ twice gives
\[
    \nabla^2 \wrisk(\M^\star) = H_{mm} - H_{mu} H_{uu}^{-1} H_{um} = H_{mm} - S_r.
\]
Since $\nabla^2 \wrisk(\M^\star) \succ 0$ by hypothesis, $0 \preceq S_r \prec H_{mm}$. Thus,
\[
    q_0 \defeq \bigl\|\tfrac{dT_r}{dm}\big|_{m^\star}\bigr\|_{H_{mm}} = \lambda_{\max}(H_{mm}^{-1/2} S_r H_{mm}^{-1/2}) < 1.
\]
Choose $\rho_r \in (q_0, 1)$ and shrink the neighborhood so that $\|\frac{dT_r}{dm}\|_{H_{mm}} \leq \rho_r$ throughout. The mean-value theorem gives
\[
    \|m^{(t)} - m^\star\|_{H_{mm}} \leq \rho_r^t \|m^{(0)} - m^\star\|_{H_{mm}}.
\]

\paragraph{Step 4: Objective-value rate}
Since $\wrisk$ is $C^\infty$ near $\M^\star$ (by the same argument as Step~4 of the proof of \cref{cor:T-iter}), with $\nabla \wrisk(\M^\star) = 0$ and $\nabla^2 \wrisk(\M^\star) \succ 0$, quadratic growth gives constants $0 < c_1 \leq c_2 < \infty$ such that
\[
    c_1\, \|m - m^\star\|_{H_{mm}}^2 \leq \wrisk(\M) - \wrisk(\M^\star) \leq c_2\, \|m - m^\star\|_{H_{mm}}^2.
\]
Combining with the iterate contraction and setting $\rho_r \defeq q_0^2 \in [0,1)$ and $C_r \defeq c_2/c_1 \geq 1$ yields the stated bound.
\end{proof}

\section{\texorpdfstring{Proof of \Cref{thm:optradius_contam}}{Proof of the contamination-radius theorem}}\label{app:optradius-contam}

Let
\begin{align}
\widehat{\Pr}_{n_1}\;\defeqq\;\frac1{n_1}\sum_{i=1}^{n_1}\delta_{X_i},
\qquad
\widehat{\Pr}_{n_2}\;\defeqq\;\frac1{n_2}\sum_{j=1}^{n_2}\delta_{Y_j},
\qquad
\widehat{\Pr}_{N}\;\defeqq\;\frac1{N}\sum_{i=1}^{n_1}\delta_{X_i}+\frac1{N}\sum_{j=1}^{n_2}\delta_{Y_j},
\end{align}
so that, with $c\defeq n_2/N\in[0,1]$,
\begin{equation}\label{eq:mix-emp}
\widehat{\Pr}_N=(1-c)\widehat{\Pr}_{n_1}+c\,\widehat{\Pr}_{n_2}.
\end{equation}

Define the events
\begin{align}
\mathcal{E}_1 &\;\defeqq\;\Big\{\Was(\Pr_0,\widehat{\Pr}_{n_1}) \le r(n_1,\eps/2)\Big\}, \\
\mathcal{E}_2 &\;\defeqq\;\Big\{\Was(\Pr_\star,\widehat{\Pr}_{n_2}) \le r(n_2,\eps/2)\Big\},
\end{align}
where $r(n,\eps)$ is the radius from Theorem~2 of \cite{fournier_rate_2015}.
By Fournier--Guillin,
\begin{equation}
\Pr_0^{\kron n_1}(\mathcal{E}_1) \geq 1 - \eps/2, \qquad \Pr_\star^{\kron n_2}(\mathcal{E}_2) \geq 1 - \eps/2.
\end{equation}
By a union bound under the joint law $\Pr_0^{\kron n_1} \otimes \Pr_\star^{\kron n_2}$,
\begin{equation}
(\Pr_0^{\kron n_1} \otimes \Pr_\star^{\kron n_2})(\mathcal{E}_1 \cap \mathcal{E}_2) \geq 1 - \eps.
\end{equation}

On $\mathcal{E}_1 \cap \mathcal{E}_2$, the triangle inequality gives
\begin{equation}
\Was(\widehat{\Pr}_{n_1}, \widehat{\Pr}_{n_2}) \leq \Was(\widehat{\Pr}_{n_1}, \Pr_0) + \Was(\Pr_0, \Pr_\star) + \Was(\Pr_\star, \widehat{\Pr}_{n_2}) \leq r(n_1, \eps/2) + D + r(n_2, \eps/2).
\end{equation}

We will use the following lemma.

\begin{lemma}[Mixture perturbation bound]\label{lem:mixture-bound}
Let $\mu,\nu$ be probability measures on $\R^d$ with finite second moments and let
$\lambda\in[0,1]$. Then
\begin{align}
\Was^2\!\big(\mu,(1-\lambda)\mu+\lambda \nu\big)\;\le\;\lambda\,\Was^2(\mu,\nu),
\qquad\text{equivalently}\qquad
\Was\!\big(\mu,(1-\lambda)\mu+\lambda \nu\big)\;\le\;\sqrt{\lambda}\,\Was(\mu,\nu).
\end{align}
\end{lemma}

\begin{proof}
Fix any coupling $\pi\in\Pi(\mu,\nu)$ between $\mu$ and $\nu$. Define a probability
measure $\widetilde{\pi}$ on $\R^d\times\R^d$ by
\begin{align}
\widetilde{\pi}\;\defeqq\;(1-\lambda)\,\pi_{\Delta} \;+\; \lambda\,\pi,
\qquad\text{where}\qquad
\pi_{\Delta}(A\times B)\defeqq \mu(A\cap B)
\end{align}
is the diagonal coupling of $\mu$ with itself. The first marginal of $\widetilde{\pi}$ is
\begin{align}
(\mathrm{proj}_1)_\#\widetilde{\pi}=(1-\lambda)(\mathrm{proj}_1)_\#\pi_\Delta+\lambda(\mathrm{proj}_1)_\#\pi
=(1-\lambda)\mu+\lambda\mu=\mu,
\end{align}
and the second marginal is
\begin{align}
(\mathrm{proj}_2)_\#\widetilde{\pi}=(1-\lambda)(\mathrm{proj}_2)_\#\pi_\Delta+\lambda(\mathrm{proj}_2)_\#\pi
=(1-\lambda)\mu+\lambda\nu.
\end{align}
Hence $\widetilde{\pi}\in\Pi\big(\mu,(1-\lambda)\mu+\lambda\nu\big)$. Therefore,
\begin{align*}
\Was^2\!\big(\mu,(1-\lambda)\mu+\lambda\nu\big)
&\le \int \norm{x-y}^2\,\mathrm{d}\widetilde{\pi}(x,y) \\
&=(1-\lambda)\int \norm{x-y}^2\,\mathrm{d}\pi_\Delta(x,y)
\;+\;\lambda\int \norm{x-y}^2\,\mathrm{d}\pi(x,y) \\
&= \lambda\int \norm{x-y}^2\,\mathrm{d}\pi(x,y),
\end{align*}
since $\pi_\Delta$ is supported on $\{(x,x)\}$ and thus has zero cost. Taking the infimum
over all $\pi\in\Pi(\mu,\nu)$ yields
$\Was^2\big(\mu,(1-\lambda)\mu+\lambda\nu\big)\le \lambda \Was^2(\mu,\nu)$.
\end{proof}

Now apply Lemma \ref{lem:mixture-bound} with $\mu=\widehat{\Pr}_{n_1}$, $\nu=\widehat{\Pr}_{n_2}$,
and $\lambda=c$. Using \eqref{eq:mix-emp} yields
\begin{equation}\label{eq:mixture-applied}
\Was\!\big(\widehat{\Pr}_{n_1},\widehat{\Pr}_N\big)
=
\Was\!\big(\widehat{\Pr}_{n_1},(1-c)\widehat{\Pr}_{n_1}+c\,\widehat{\Pr}_{n_2}\big)
\le \sqrt{c}\,\Was\!\big(\widehat{\Pr}_{n_1},\widehat{\Pr}_{n_2}\big).
\end{equation}
On $\mathcal{E}_1 \cap \mathcal{E}_2$, plugging into \eqref{eq:mixture-applied} gives the high-probability bound
\begin{equation}\label{eq:det-perturb}
\Was\!\big(\widehat{\Pr}_{n_1},\widehat{\Pr}_N\big)\le \sqrt{c}\bigl(D + r(n_1,\eps/2) + r(n_2,\eps/2)\bigr).
\end{equation}

On the event $\mathcal{E}_1 \cap \mathcal{E}_2$, by the triangle inequality and \eqref{eq:det-perturb},
\begin{align}
\Was(\Pr_0,\widehat{\Pr}_N)
\le \Was(\Pr_0,\widehat{\Pr}_{n_1}) + \Was(\widehat{\Pr}_{n_1},\widehat{\Pr}_N)
\le r(n_1,\eps/2) + \sqrt{c}\bigl(D + r(n_1,\eps/2) + r(n_2,\eps/2)\bigr).
\end{align}
Therefore,
\begin{align}
&(\Pr_0^{\kron n_1}\!\otimes\! \Pr_\star^{\kron n_2})
\cl{\Was(\Pr_0,\widehat{\Pr}_N)\le r(n_1,\eps/2)+\sqrt{c}\bigl(D + r(n_1,\eps/2) + r(n_2,\eps/2)\bigr)} \nonumber\\
&\quad\;\ge\;
(\Pr_0^{\kron n_1}\otimes \Pr_\star^{\kron n_2})(\mathcal{E}_1 \cap \mathcal{E}_2)
\;\ge\;1-\eps,
\end{align}
where the final inequality follows from the union bound. This proves the claim with
$r_{\mathrm{cont}}(N,\eps,c)\defeq r(n_1,\eps/2)+\sqrt{c}\bigl(D + r(n_1,\eps/2) + r(n_2,\eps/2)\bigr)$.

\section{Additional Algorithms}\label{app:algs}

Let $H(\bm\pi) \triangleq -\sum_{k=1}^{K} \pi_k \log \pi_k$ be the entropy of a pmf $\bm\pi\in \Delta_K$. For $\lambda>0$, define the entropy-regularized objective
\begin{align}
    J_{\gamma}^\lambda(\M,\bm{\Pi}) \triangleq \gamma r^2 +\frac{1}{N}  \sum_{n=1}^{N}\br{\sum_{k=1}^{K} \norm{x_n -\mu_k}^2 \pi_{kn} + \frac{1}{\gamma-1} \norm{x_n-\M\bm{\pi_n}}^2 - \lambda H(\bm\pi_n)}.
\end{align}
Note that the entropy term $-\lambda H(\bm\pi_n) = \lambda \sum_k \pi_{kn} \log \pi_{kn}$ makes $J_\gamma^\lambda$ strictly convex in $\bm\Pi$ (on the simplex), while the $\M$-block remains unchanged from the unregularized case. For better convergence properties, we propose the following alternating minimization variant of \cref{alg:DR lloyd max}.
\begin{algorithm}[ht]
   \caption{DR-$K$-Means Clustering with Strongly Convex Regularizers}
   \label{alg:DR lloyd max regularizers}
\begin{algorithmic}[1]
   \STATE {\bfseries input:}  $\gamma\>1$, dataset $\DN$, initialization $\M^{(0)}$, convergence tolerance $\varepsilon\>0$,
   \REPEAT 
   \STATE $\begin{aligned}\bm{\Pi}^{(t)} &\gets \argmin_{\bm{\pi}_1,\dots,\bm{\pi}_N\in \Delta_K}  J_{\gamma}^\lambda(\M^{(t)},\bm{\Pi}), \end{aligned}$
  \STATE $\begin{aligned}\M^{(t+1)} &\gets \argmin_{\M\in \R^{d\times K}}J_{\gamma}^\lambda(\M,\bm{\Pi}^{(t)})\end{aligned}$ 
   \STATE Increment $t\gets t+1$
\UNTIL{$\begin{aligned}\max_{k}\norm{\mu_k^{(t+1)} \!-\! \mu_k^{(t)}}/ \max_{l}\norm{ \mu_l^{(t)}} \leq \varepsilon \end{aligned}$}
   \STATE {\bfseries return} $\M^{(t)}$
\end{algorithmic}
\end{algorithm}

\cref{alg:DR lloyd max regularizers} requires exact block minimization of $J_{\gamma}^\lambda(\M,\bm{\Pi})$. For the $\bm\Pi$-step, the entropy regularizer does not admit a simple closed form (due to coupling through $\|x_n - \M\bm\pi_n\|^2$), so an iterative inner solver is needed.

\begin{theorem}[Linear convergence with entropy regularization]\label{thm:reg-convergence}
Fix $\gamma>1$ and $\lambda > 0$. Let $\M^\star_\lambda$ be a non-degenerate strict local minimizer of $F_\gamma^\lambda(\M) \defeq \min_{\bm{\Pi}\in\Delta_K^{\times N}} J_\gamma^\lambda(\M,\bm{\Pi})$, \ie, $\nabla^2 F_\gamma^\lambda(\M^\star_\lambda) \succ 0$, with all clusters non-empty. Then there exist $\delta > 0$, $\rho_\lambda \in [0,1)$, and $C_\lambda \geq 1$ such that for all $\norm{\M^{(0)} - \M^\star_\lambda}\leq \delta$, the iterates of \cref{alg:DR lloyd max regularizers} satisfy
\begin{equation}
    F_\gamma^\lambda(\M^{(T)}) - F_\gamma^\lambda(\M^\star_\lambda) \leq C_\lambda\,\rho_\lambda^T\bigl(F_\gamma^\lambda(\M^{(0)}) - F_\gamma^\lambda(\M^\star_\lambda)\bigr).
\end{equation}
Furthermore, the approximation bias satisfies $\abs{F_\gamma^\lambda(\M) - F_\gamma(\M)} \leq \lambda \log K$ for all $\M$.

Compared to \cref{cor:T-iter}, the entropy regularization provides the following advantages:
\begin{enumerate}
    \item \textbf{No affine independence required}: the entropy forces $\pi^\star_{kn} > 0$ for all $k,n$, making $R \succ 0$ regardless of centroid geometry (the entropy contributes $\frac{\lambda}{N}\diag(\pi_{kn}^{-1})$ to the $\bm\Pi$-Hessian).
    \item \textbf{No strict complementarity required}: the entropy pushes all assignments to the relative interior of $\Delta_K$, so strict complementarity is vacuously satisfied.
    \item \textbf{SOSC reduces to envelope non-degeneracy}: by \cref{thm:sosc-equivalence}, the SOSC is equivalent to $\nabla^2 F_\gamma^\lambda(\M^\star_\lambda) \succ 0$, which holds generically (\cref{rem:geometric-sosc}).
\end{enumerate}
\end{theorem}

\begin{proof}\mbox{}

\paragraph{Step 1: Interior minimizer and constraint qualification}
The logarithmic term from $-\lambda H(\bm{\pi}_n)$ forces the minimizer $\bm{\Pi}^\star_\lambda$ to the relative interior: $\pi^\star_{kn} > 0$ for all $k,n$. Consequently, no inequality constraints $\pi_{kn} \geq 0$ are active ($A_n = [K]$ for all $n$), and strict complementarity is vacuously satisfied (there are no binding inequality constraints whose multipliers need checking). LICQ requires the gradients of the active constraints to be linearly independent; the only active constraints are the equalities $\sum_k \pi_{kn} = 1$, whose gradient is $\mathbf{1}_K$---a single nonzero vector, which forms a linearly independent set by itself.

\paragraph{Step 2: SOSC via \cref{thm:sosc-equivalence}}
The restricted $\bm\Pi$-Hessian on the tangent space $\mathcal{T} = \{\delta\bm\Pi : \sum_k \delta\pi_{kn} = 0\}$ is
\[
    R_\lambda = \frac{2}{N(\gamma-1)} \operatorname{blkdiag}\bigl(U_n^T \M^{\star T} \M^\star U_n\bigr)_{n=1}^N + \frac{\lambda}{N}\operatorname{blkdiag}\bigl(U_n^T\diag(\pi_{kn}^{\star\,-1})U_n\bigr)_{n=1}^N.
\]
The entropy contribution satisfies $U_n^T\diag(\pi_{kn}^{\star\,-1})U_n \succcurlyeq (\max_k \pi^\star_{kn})^{-1}I_{K-1} \succ 0$ since all $\pi^\star_{kn} > 0$. Hence $R_\lambda \succ 0$ \emph{regardless} of centroid geometry---no affine independence needed.

By \cref{thm:sosc-equivalence}, $P^\star_\lambda \succ 0 \iff \nabla^2 F_\gamma^\lambda(\M^\star_\lambda) \succ 0$, which is our hypothesis.

\paragraph{Step 3: Convergence via direct contraction}
With $R_\lambda \succ 0$ (Step~2) and $\nabla^2 F_\gamma^\lambda(\M^\star_\lambda) \succ 0$ (hypothesis), the proof of \cref{cor:T-iter} (Steps~3--4) applies verbatim with $J_\gamma$ replaced by $J_\gamma^\lambda$. In particular, the entropy term $-\lambda H(\bm\pi_n)$ is $C^\infty$ on the interior of $\Delta_K$, and Step~1 ensures the minimizers lie in the interior, so the IFT and all subsequent arguments carry through. This yields the stated iterate and objective-value bounds with $C_\lambda = c_2/c_1$ as in the proof of \cref{cor:T-iter}.

\paragraph{Step 4: Approximation bias}
Since $H(\bm\pi_n) \geq 0$, subtracting the entropy term can only decrease the surrogate: $J_\gamma^\lambda(\M,\bm\Pi) \leq J_\gamma(\M,\bm\Pi)$. Taking the infimum over $\bm\Pi$ gives $F_\gamma^\lambda(\M) \leq F_\gamma(\M)$. Conversely, $H(\bm\pi_n) \leq \log K$ gives $J_\gamma^\lambda(\M,\bm\Pi) \geq J_\gamma(\M,\bm\Pi) - \lambda \log K$, so $F_\gamma^\lambda(\M) \geq F_\gamma(\M) - \lambda \log K$. Together,
\[
    0 \;\leq\; F_\gamma(\M) - F_\gamma^\lambda(\M) \;\leq\; \lambda \log K \qquad \forall\, \M.
\]

\end{proof}

\section{Interpretation of the stochastic assignment probabilities} \label{app:lemma_pi}
\begin{lemma}[Closed-form saddle point in the scalar setting]\label{thm:opt pi}
Assume $d=1$, fix $\gamma>1$, and let the centroids satisfy
\[
\mu_1<\mu_2<\cdots<\mu_K.
\]
Set $\mu_0:=-\infty$, $\mu_{K+1}:=+\infty$, and
\[
\mu_{k+\half}:=\frac{\mu_k+\mu_{k+1}}{2},\qquad k=1,\dots,K-1.
\]
Define
\begin{align}
\Phi_{\gamma,k}
&:=
\mu_k
+
\left(1-\gamma^{-1}\right)
\left[
\frac{\mu_{k-1}-\mu_k}{2},
\frac{\mu_{k+1}-\mu_k}{2}
\right],
\qquad k=1,\dots,K, \label{eq: phi_k}\\[0.5em]
\Phi_{\gamma,k+\half}
&:=
\mu_{k+\half}
+
\gamma^{-1}
\left[
-\frac{\mu_{k+1}-\mu_k}{2},
\frac{\mu_{k+1}-\mu_k}{2}
\right],
\qquad k=1,\dots,K-1. \label{eq: phi midpoint}
\end{align}
For a sample $x_n$, define
\begin{equation}\label{eq:q-ratio}
q_{k+\half,n}
:=
\frac{x_n-\mu_{k+\half}}{(\mu_{k+1}-\mu_k)/2},
\qquad k=1,\dots,K-1.
\end{equation}
Then the saddle point $(x_n^\star,\pi_n^\star)$ from \cref{thm:inner saddle} satisfies
\begin{equation}\label{eq:xstar-scalar}
x_n^\star=
\begin{cases}
x_n+\dfrac{x_n-\mu_k}{\gamma-1},
& x_n\in \Phi_{\gamma,k},\\[0.8em]
\mu_{k+\half},
& x_n\in \Phi_{\gamma,k+\half},
\end{cases}
\end{equation}
and the $k$-th coordinate of $\pi_n^\star$ is
\begin{equation}\label{eq:pistar-scalar}
\pi_{kn}^\star=
\begin{cases}
\dfrac12\bigl(1+\gamma q_{k-\half,n}\bigr),
& x_n\in \Phi_{\gamma,k-\half},\\[0.8em]
1,
& x_n\in \Phi_{\gamma,k},\\[0.8em]
\dfrac12\bigl(1-\gamma q_{k+\half,n}\bigr),
& x_n\in \Phi_{\gamma,k+\half},
\end{cases}
\qquad k=1,\dots,K,
\end{equation}
with the convention that $\Phi_{\gamma,\half}=\Phi_{\gamma,K+\half}=\varnothing$.
\end{lemma}

\begin{proof}
Fix a scalar point $x\in\R$. By \cref{thm:inner saddle}, a saddle point $(x^\star,\pi^\star)$ satisfies
\begin{align}
x^\star
&=
x+\frac{x-\sum_{j=1}^K \pi_j^\star \mu_j}{\gamma-1},
\label{eq:scalar-xstar}\\
\pi^\star
&\in
\argmin_{\pi\in\Delta_K}
\sum_{j=1}^K (x^\star-\mu_j)^2\pi_j.
\label{eq:scalar-pistar}
\end{align}

Because the objective in \eqref{eq:scalar-pistar} is linear in $\pi$, every minimizer is supported on the set of centroids that are closest to $x^\star$. Since
\[
\mu_1<\mu_2<\cdots<\mu_K,
\]
a scalar point can have at most two nearest centroids, and if there are two of them, then they must be consecutive.

We consider the two possible cases.

\paragraph{Case 1: a unique nearest centroid}
Suppose that $x^\star$ is strictly closer to $\mu_k$ than to all other centroids. Then
\[
\pi^\star=\delta_k.
\]
Substituting this into \eqref{eq:scalar-xstar} yields
\[
x^\star=x+\frac{x-\mu_k}{\gamma-1}.
\]
The centroid $\mu_k$ is the unique nearest one precisely when $x^\star$ lies in the Voronoi cell of $\mu_k$, i.e.,
\[
\mu_{k-\half}<x^\star<\mu_{k+\half},
\qquad
\mu_{k\pm\half}:=\frac{\mu_k+\mu_{k\pm1}}{2}.
\]
Substituting the expression for $x^\star$ and solving for $x$ gives
\[
\mu_k+\left(1-\gamma^{-1}\right)\frac{\mu_{k-1}-\mu_k}{2}
<
x
<
\mu_k+\left(1-\gamma^{-1}\right)\frac{\mu_{k+1}-\mu_k}{2}.
\]
This is exactly the interior of $\Phi_{\gamma,k}$. Hence, for $x\in\Phi_{\gamma,k}$,
\[
x^\star=x+\frac{x-\mu_k}{\gamma-1},
\qquad
\pi^\star=\delta_k.
\]

\paragraph{Case 2: two nearest centroids}
Suppose now that $x^\star$ has two nearest centroids. Then they must be $\mu_k$ and $\mu_{k+1}$, and necessarily
\[
x^\star=\mu_{k+\half}:=\frac{\mu_k+\mu_{k+1}}{2}.
\]
In this case,
\[
\pi^\star=p\,\delta_k+(1-p)\,\delta_{k+1}
\qquad\text{for some }p\in[0,1].
\]
Substituting this into \eqref{eq:scalar-xstar} gives
\[
\mu_{k+\half}
=
x+\frac{x-(p\mu_k+(1-p)\mu_{k+1})}{\gamma-1}.
\]
Rearranging,
\begin{align*}
\gamma(\mu_{k+\half}-x)
&=
\mu_{k+\half}-(p\mu_k+(1-p)\mu_{k+1})\\
&=
\left(p-\frac12\right)(\mu_{k+1}-\mu_k).
\end{align*}
Therefore,
\[
p
=
\frac12-\frac{\gamma(x-\mu_{k+\half})}{\mu_{k+1}-\mu_k}
=
\frac12\Bigl(1-\gamma q_{k+\half}(x)\Bigr),
\]
where
\[
q_{k+\half}(x)
:=
\frac{x-\mu_{k+\half}}{(\mu_{k+1}-\mu_k)/2}.
\]
The constraint $p\in[0,1]$ is equivalent to
\[
|x-\mu_{k+\half}|
\le
\gamma^{-1}\frac{\mu_{k+1}-\mu_k}{2},
\]
which is precisely the condition $x\in\Phi_{\gamma,k+\half}$. Hence, for $x\in\Phi_{\gamma,k+\half}$,
\[
x^\star=\mu_{k+\half},
\qquad
\pi^\star=p\,\delta_k+(1-p)\,\delta_{k+1},
\qquad
p=\frac12\Bigl(1-\gamma q_{k+\half}(x)\Bigr).
\]

Applying the above characterization to $x=x_n$ yields \eqref{eq:xstar-scalar}. The formula \eqref{eq:pistar-scalar} follows by reading off the weight assigned to centroid $\mu_k$: it equals $1$ on $\Phi_{\gamma,k}$, equals
\[
\frac12\bigl(1-\gamma q_{k+\half,n}\bigr)
\]
on $\Phi_{\gamma,k+\half}$, and equals
\[
\frac12\bigl(1+\gamma q_{k-\half,n}\bigr)
\]
on $\Phi_{\gamma,k-\half}$.
\end{proof}

\section{Additional Simulations}\label{app:exp}
In the experiments below, we calibrate the ambiguity radius via the high-confidence formula of \cref{thm:optradius}, setting $r = C \cdot N^{-1/d}$ for a constant $C > 0$. The corresponding dual parameter $\gamma^\star$ is then determined by \eqref{eq: opt gamma}; equivalently, fixing $\gamma > 1$ implicitly selects a radius through the same relationship.

In Experiment 1, we consider a Gaussian mixture model (GMM) with 3 components and weights $\{0.2, 0.26, 0.53\}$ and randomly chosen centers in $d = 7$ dimensions. We set $r = 10\left(\frac{1}{N}\right)^{\frac{1}{d}}$ and $k = 3$. We observe in \cref{fig:wce} that as the number of data points $N$ increases from $5$ to $30$, the difference in the worst-case performance between DRKM and KM decreases. This is because a higher $N$ better represents the underlying data distribution, making it easier for KM to perform.

\begin{figure}[t]
    \centering
    \includegraphics[width=0.5\linewidth]{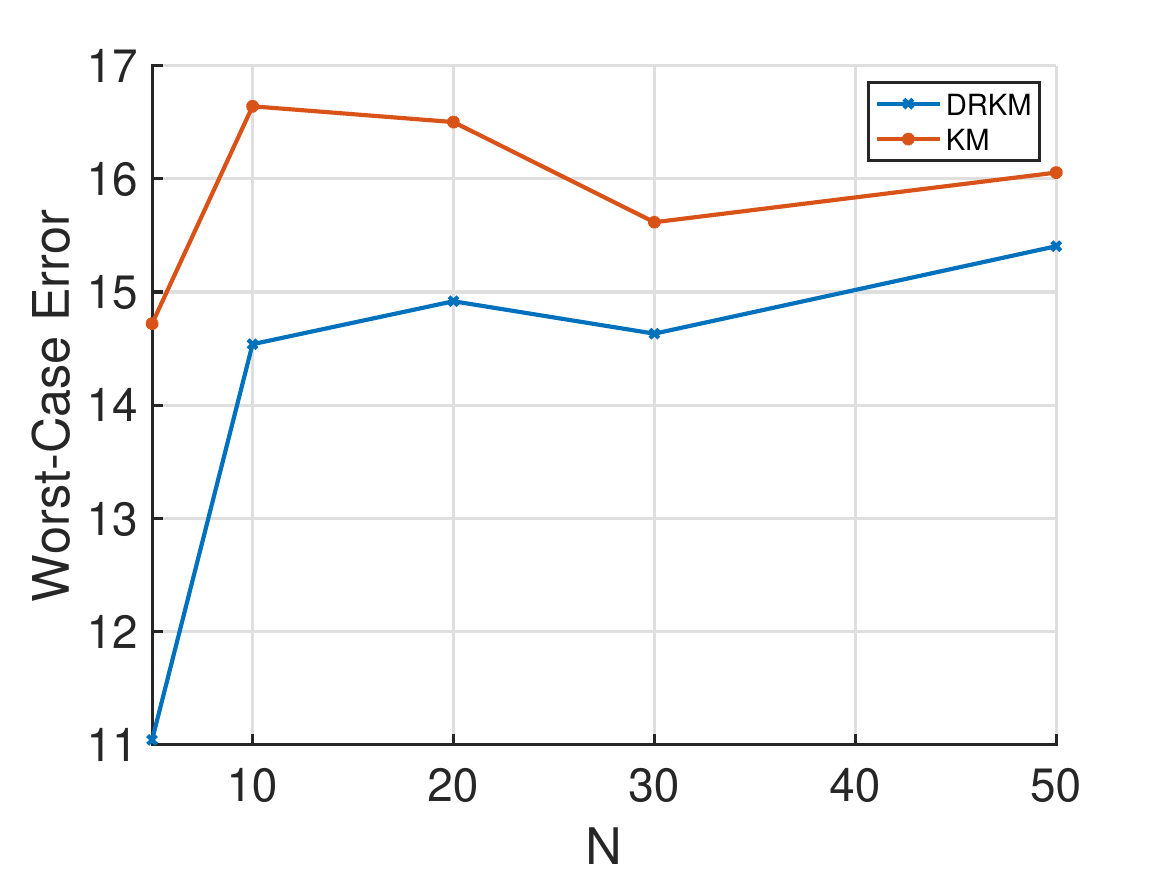}
    \caption{The trend of the worst-case error in Experiment 1, averaged over 30 trials for each $N$. We see that as $N$ increases, the difference in the performance of DRKM and KM decreases, signifying the performance improvement of DRKM over KM in the low-data regime.}
    \label{fig:wce}
\end{figure}

In Experiment 2, we consider a dataset that is corrupted by an outlier cluster. The original dataset, as represented in \cref{fig:class}, consists of two clusters (Cluster A and Cluster B) that are normally distributed with means $(-2, -2)$ and $(2, 2)$ respectively. The outlier cluster is also a set of normally distributed points centered at $(8, 8)$. The misclassification error of DRKM and KM, averaged over 200 trial runs, is given in \cref{tab:1}. We see that DRKM performs better than KM in the low-data regime.

\begin{figure}[t]
    \centering
    \includegraphics[width=0.5\linewidth]{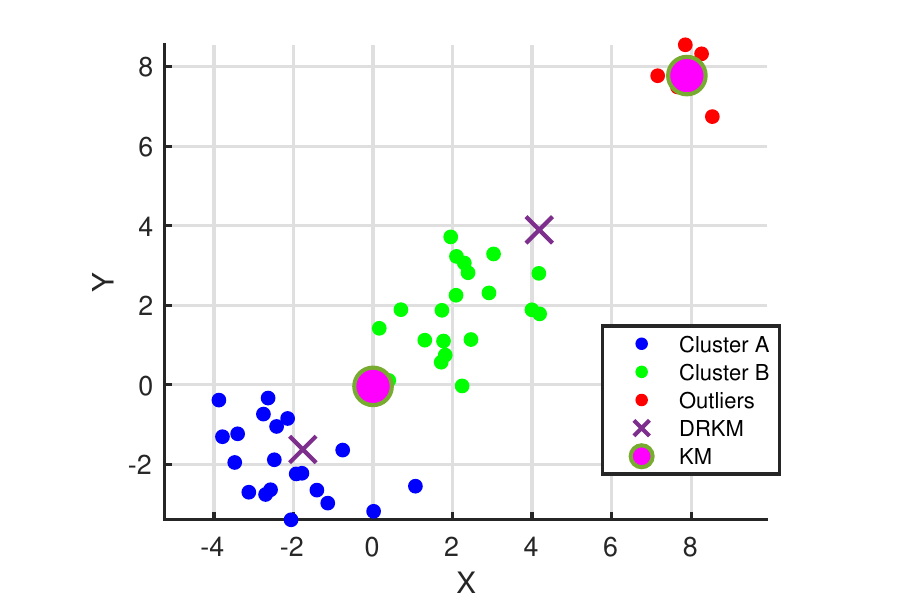}
    \caption{A realization of the dataset in Experiment 2. The blue points are in Cluster A and the green points are in Cluster B. We see that the cluster center of KM is shifted to the outlier cluster, whereas the cluster center of DRKM is close to the true cluster (Cluster B) even in the presence of an outlier cluster and when we initialize DRKM with the output of KM.}
    \label{fig:class}
\end{figure}

In Experiment 3, we consider the worst-case performance on the Shuttle dataset of the UCI repository \cite{uci}. The original dataset consists of $43,500$ data points with $d = 9$ features. We sparsify the dataset to study the worst-case performance of DRKM and KM in the low-data regime.
 We choose the ambiguity radius as $r = 100\left(\frac{1}{N}\right)^{\frac{1}{d}}$. The results are shown in \cref{fig:45k}. As in \cref{fig:wce}, as the number of data points $N$ increases from $5$ to $30$, the difference in the worst-case performance between DRKM and KM decreases.

\begin{figure}[t]
    \centering
    \includegraphics[width=0.5\linewidth]{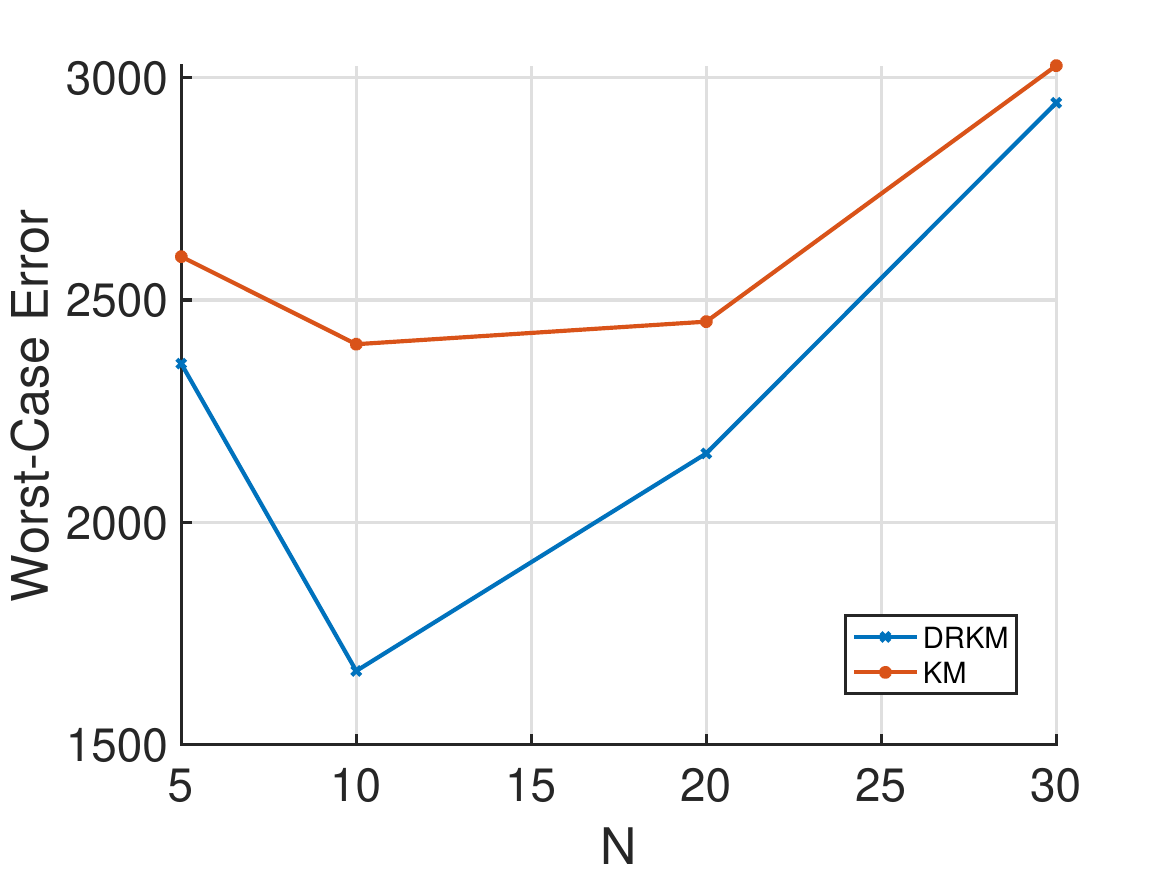}
    \caption{The trend of the worst-case error in Experiment 3, averaged over 30 trials for each $N$. We randomly pick $N$ data points from the dataset in our experiment. We see that as $N$ increases, the difference in the performance of DRKM and KM decreases, signifying the performance improvement of DRKM over KM in the low-data regime.}
    \label{fig:45k}
\end{figure}
\begin{table}[H]
    \caption{Classification Accuracy for Experiment 2}
    \label{tab:1}
    \centering
    \subfloat[We compare DRKM++ and KM++ for different data samples and radii of the ambiguity set. DRKM noticeably outperforms KM.]{%
      \begin{minipage}[t]{\linewidth}
        \centering
        \begin{tabular}{@{}c|c|c|c|c|c @{}}
            \toprule
              $N_A$ & $N_B$ & $N_o$ & Accuracy  & Accuracy  & r \\
               &  &  &   (DRKM) &   (KM) &  \\
            \midrule
             20 & 20 & 5 & \textbf{0.86} & 0.78 & 2.25 \\
             \midrule
             10 & 10 & 2 & \textbf{0.92} & 0.8  & 2.5   \\
            \bottomrule
        \end{tabular}
      \end{minipage}}
\end{table}

\begin{table}[H]
\caption{Classification Accuracy for the CIFAR-10 Dataset}
\label{tab:cifar_acc}
\centering
\subfloat[$N$ samples are reduced to $d = 256$ dimensions using ResNet-18 and PCA. We corrupt $50\%$ of the points with independent Gaussian noise. Despite the severe lack of data, DRKM++ learns the true distribution better than KM++, even though it is seeded by KM++ ($\gamma = 15$).]{%
  \begin{minipage}[t]{\linewidth}
    \centering
    \setlength{\tabcolsep}{6pt}
    \begin{tabular}{@{}l r | ccc @{}}
    \toprule
    \multirow{2}{*}{Method}  & \multirow{2}{*}{$N$} &
      \multicolumn{3}{c}{Accuracy}  \\ 
       & & Max. & Avg. & Min.  \\ 
    \midrule
    \multirow{1}{*}{KM\texttt{++}}
       & 5000 & 0.30 & 0.22 & 0.11 \\
        \midrule
    \multirow{1}{*}{\textbf{DRKM\texttt{++}}}
         & 5000 & \textbf{0.31} & \textbf{0.24} & \textbf{0.18} \\
    \bottomrule
    \end{tabular}
  \end{minipage}}
\end{table}





\end{document}